\documentclass[11pt, oneside]{article}
\usepackage{etoolbox}
\usepackage{graphicx}
\usepackage{amsmath,amsthm,amssymb,amsfonts,latexsym,amscd,euscript}
\usepackage{bbm}
\usepackage{framed}
\usepackage{fullpage}
\usepackage{mathrsfs}
\usepackage{physics}
\usepackage[hypertexnames=false]{hyperref}
\usepackage{enumitem}

\usepackage[margin=1in]{geometry}

\hypersetup{
  colorlinks=true,
  linkcolor=blue,
  citecolor=blue,
  urlcolor=black
}
\usepackage{accents}
\usepackage{setspace}
\usepackage{tikz-cd}
\newcount\Comments  %
\ifnum\Comments=1
\usepackage[inline]{showlabels}

\fi
\usepackage{turnstile}
\usepackage{mathpartir}
\usepackage{tikz}
\usepackage{xspace}

 \usepackage{tcolorbox}

\usepackage{algorithm}
\usepackage[noend]{algpseudocode}

\newcommand{\multiline}[1]{\parbox[t]{\dimexpr\linewidth-\algorithmicindent}{#1}}

\newcommand{\nc}{\newcommand}

\usepackage[nameinlink,capitalize]{cleveref}
\crefformat{equation}{#2(#1)#3}
\Crefformat{equation}{#2(#1)#3}

\Crefformat{figure}{#2Figure #1#3}
\Crefname{assumption}{Assumption}{Assumptions}
\Crefformat{assumption}{#2Assumption #1#3}
   \Crefname{question}{Question}{Questions}
   \Crefformat{question}{#2Question #1#3}
   \Crefname{claim}{Claim}{Claims}
   \Crefformat{claim}{#2Claim #1#3}
   \Crefname{problem}{Problem}{Problems}
  \Crefformat{problem}{#2Problem #1#3}
  \Crefname{idea}{Main Idea}{Main Ideas}
  \Crefformat{idea}{#2Main Idea #1#3}
\Crefname{subsubsection}{Section}{Sections}
\crefformat{subsubsection}{#2Section #1#3}
\Crefformat{subsubsection}{#2Section #1#3}

\nc{\sups}[1]{^{\scriptscriptstyle{#1}}}
\nc{\subs}[1]{_{\scriptscriptstyle{#1}}}

\newcommand{\wb}{\widebar}

\nc{\Critic}{\texttt{Critic}\xspace}
\nc{\PSDPUCB}{\texttt{PSDP-UCB}\xspace}
\nc{\LSVIUCB}{\texttt{LSVI-UCB}\xspace}
\nc{\Actor}{\texttt{Actor}\xspace}
\nc{\EstFeature}{\texttt{EstFeature}\xspace}
\nc{\ExpFTPL}{\texttt{ExpFTPL}\xspace}
\nc{\dist}{\mathrm{dist}}
\nc{\Bquad}{B^{\mathsf{quad}}}

\newcommand{\R}{\mathbb{R}}

\newcommand{\calF}{\mathcal{F}}
\newcommand{\calH}{\mathcal{H}}

\makeatletter
\newtheorem*{rep@theorem}{\rep@title}
\newcommand{\newreptheorem}[2]{%
\newenvironment{rep#1}[1]{%
 \def\rep@title{#2 \ref{##1}}%
 \begin{rep@theorem}}%
 {\end{rep@theorem}}}
\makeatother

\makeatletter
\newcommand\xlabel[2][]{\phantomsection\def\@currentlabelname{#1}\label{#2}}
\makeatother

{
\theoremstyle{plain}
\newtheorem{theorem}{Theorem}
\newtheorem{idea}[theorem]{Main Idea}
\newtheorem{lemma}[theorem]{Lemma}
\newtheorem{corollary}[theorem]{Corollary}

\newtheorem{fact}[theorem]{Fact}

\newtheorem{claim}[theorem]{Claim}

\theoremstyle{definition}
\newtheorem{definition}{Definition}

\newtheorem{remark}[definition]{Remark}

\newtheorem{problem}[definition]{Problem}

\numberwithin{theorem}{section}
\numberwithin{definition}{section}
}

\nc{\DMO}{\DeclareMathOperator}

\DMO{\prox}{prox}
\DMO{\UCB}{UCB}
\DMO{\LCB}{LCB}
\nc{\phidiff}{\phi\sups{\Delta}}
\nc{\pexp}{q_{\mathrm{exp}}}
\nc{\nn}{\nonumber}
\nc{\rk}{\mathrm{rk}}
\nc{\brk}[3]{{\rm br}_{#1}^{#2}({#3})}
\nc{\co}{{\rm co}}
\nc{\br}[2]{{\rm br}^{#1}({#2})}
\nc{\depth}[1]{{\rm d}({#1})}
\nc{\tA}{\textsc{A}}
\nc{\child}[2]{{\rm ch}_{#1}({#2})}
\nc{\parent}[1]{{\rm pa}({#1})}
\nc{\dg}{\dagger}
\nc{\bB}{\mathbf{B}}
\nc{\Span}{\mathsf{span}}
\nc{\unif}{\mathsf{unif}}
\nc{\indsig}[2]{\mathcal{I}_{#1}({#2})}
\nc{\total}{{\rm fin}}
\nc{\early}{{\rm pre}}
\nc{\zsink}{z_{\rm sink}}
\nc{\lowv}{{\rm low}}
\nc{\ol}{\overline}
\nc{\ul}{\underline}
\nc{\madec}[3]{\texttt{ma-dec}_{#1}({#2}, {#3})}
\nc{\madeco}[1]{\texttt{ma-dec}_{#1}}
\nc{\madecd}[3]{\texttt{ma-dec}^{\texttt{d}}_{#1}({#2}, {#3})}
\nc{\SF}{\mathscr{F}}
\nc{\SH}{\mathscr{H}}
\nc{\SP}{\mathscr{P}}
\nc{\SPc}{\wb{\mathscr{P}}}
\nc{\SB}{\mathscr{B}}
\nc{\SC}{\mathscr{C}}
\nc{\BS}{\mathbb{S}}
\nc{\PiMarkov}{\Pi^{\rm markov}}
\nc{\trunc}[2]{\mathsf{trunc}_{#2}({#1})}
\nc{\sbl}{of strong Bellman type\xspace}
\nc{\inormal}[1][\Phi, u,v]{\til{N}_{{#1}}}

\nc{\gamvec}{\gamma}
\nc{\til}{\widetilde}
\nc{\td}{\tilde}
\nc{\wh}{\widehat}

\nc{\old}[1]{\ifnum\Comments=1 {\color{brown}  [OLD: #1]}\fi}
\nc{\noah}[1]{\ifnum\Comments=1 {\color{purple} [ng: #1]}\fi}
\nc{\allen}[1]{\ifnum\Comments=1 {\color{red} [al: #1]}\fi}
\nc{\shetty}[1]{\ifnum\Comments=1 {\color{blue} [avs: #1]}\fi}

\nc{\BP}{\mathbb{P}}
\nc{\BI}{\mathbb{I}}
\nc{\midpoint}[1][\Phi,\phi_1,\phi_2]{\mu^{\star}_{{#1}}}

\nc{\fools}[3]{\MF_{#3}({#1}, {#2})}
\nc{\fool}[2]{\MF({#1},{#2})}
\nc{\clip}[2]{{\rm clip}\left[ \left. {#1} \right| {#2} \right]}
\nc{\imax}{\omega}
\DMO{\conv}{conv}
\nc{\MH}{\mathcal{H}}
\nc{\MV}{\mathcal{V}}
\nc{\MC}{\mathcal{C}}
\nc{\MI}{\mathcal{I}}
\nc{\st}{\star}
\nc{\lng}{\langle}
\nc{\rng}{\rangle}
\DMO{\OOPT}{opt}
\nc{\dopt}[2]{\ell_{\OOPT}({#1},{#2})}
\nc{\MG}{\mathcal{G}}
\nc{\MP}{\mathcal{P}}
\nc{\PP}{\mathbb{P}}
\nc{\TT}{\mathbb{T}}
\nc{\TTmax}{\TT_{\max}}
\DMO{\REG}{Reg}
\DMO{\WREG}{wReg}
\nc{\reg}[2]{{\Delta}_{{#1}}({#2})}
\nc{\wreg}[2]{{\Delta}^{\rm w}_{{#1}}({#2})}
\nc{\Reg}[2]{{\REG}_{{#1}}({#2})}
\nc{\wReg}[2]{{\WREG}_{{#1}}({#2})}
\DMO{\Gap}{Gap}
\DMO{\GD}{GD}
\DMO{\GDA}{GDA}
\DMO{\EG}{EG}
\nc{\TE}{\til{\E}}
\nc{\Var}{\mathbb{V}}
\DMO{\Cov}{Cov}
\DMO{\OGDA}{OGDA}
\DMO{\Unif}{Unif}
\nc{\Qu}{\ul{Q}}
\nc{\Qo}{\ol{Q}}
\nc{\Ro}{\ol{R}}
\nc{\Vu}{\ul{V}}
\nc{\Vo}{\ol{V}}
\nc{\RanQ}{\Delta Q}
\nc{\RanV}{\Delta V}
\nc{\clipQ}{\Delta \breve{Q}}
\nc{\frzQ}{\Delta \mathring{Q}}
\nc{\clipV}{\Delta \breve{V}}
\nc{\clipdelta}{\breve{\delta}}
\nc{\cliptheta}{\breve{\theta}}
\nc{\delmin}{\Delta_{{\rm min}}}
\nc{\delmins}[1]{\Delta_{{\rm min},{#1}}}
\nc{\gapfinal}[1]{\max \left\{ \frac{\frzQ_{{#1}}^{k^\st}(x,a)}{2H}, \frac{\delmin}{4H} \right\}}
\nc{\post}[2]{R({#1}; {#2})}
\nc{\posts}[3]{R_{#3}({#1}; {#2})}

\nc{\algnst}[1]{\begin{align*}#1\end{align*}}
\nc{\algn}[1]{\begin{align}#1\end{align}}
\nc{\matx}[1]{\left(\begin{matrix}#1\end{matrix}\right)}
\renewcommand{\^}[1]{^{(#1)}}

\nc{\nuu}{\nu}

\nc{\bel}[1]{\mathbf{b}({#1})}
\nc{\nbel}[1]{\bar{\mathbf{b}}({#1})}
\nc{\sbel}[2]{\mathbf{b}'_{#1}({#2})}
\nc{\nsbel}[2]{\bar{\mathbf{b}}'_{#1}({#2})}

\nc{\bv}{\mathbf{v}}
\nc{\bone}{\mathbf{1}}
\nc{\bX}{\mathbf{X}}
\nc{\bY}{\mathbf{Y}}
\nc{\bG}{\mathbf{G}}
\nc{\bz}{\mathbf{z}}
\nc{\bw}{\mathbf{w}}
\nc{\bA}{\mathbf{A}}
\nc{\bJ}{\mathbf{J}}
\nc{\bK}{\mathbf{K}}
\nc{\bb}{\mathbf{b}}
\nc{\ba}{\mathbf{a}}
\nc{\bc}{\mathbf{c}}
\nc{\bC}{\mathbf{C}}
\nc{\BR}{\mathbb R}
\nc{\BA}{\mathbb{A}}
\nc{\BC}{\mathbb C}
\nc{\bx}{\mathbf{x}}
\nc{\bS}{\mathbf{S}}
\nc{\bM}{\mathbf{M}}
\nc{\bR}{\mathbf{R}}
\nc{\bN}{\mathbf{N}}
\nc{\NN}{\mathbb{N}}
\nc{\by}{\mathbf{y}}
\nc{\sy}{y}
\nc{\sx}{x}

\nc{\MO}{\mathcal O}
\nc{\MU}{\mathcal{U}}
\nc{\ME}{\mathcal{E}}
\nc{\MN}{\mathcal{N}}
\nc{\MK}{\mathcal{K}}
\nc{\MM}{\mathcal{M}}
\nc{\MS}{\mathcal{S}}
\nc{\MT}{\mathcal{T}}
\nc{\BF}{\mathbb F}
\nc{\BQ}{\mathbb Q}
\nc{\MX}{\mathcal{X}}
\nc{\MA}{\mathcal{A}}
\nc{\MD}{\mathcal{D}}
\nc{\MB}{\mathcal{B}}
\nc{\MZ}{\mathcal{Z}}
\nc{\MJ}{\mathcal{J}}
\nc{\MW}{\mathcal{W}}
\nc{\MR}{\mathcal{R}}
\nc{\MY}{\mathcal{Y}}
\nc{\BZ}{\mathbb Z}
\nc{\BN}{\mathbb N}
\nc{\ep}{\epsilon}
\nc{\epbe}{\varepsilon_{\mathsf{BE}}}
\nc{\epout}{\varepsilon_{\mathsf{outlier}}}
\nc{\bellc}[1][h]{\MT_{#1}^\circ}
\nc{\vep}{\varepsilon}
\nc{\gapfn}[1]{\varepsilon_{#1}}
\nc{\ggapfn}[2]{\varphi_{#1}({#2})}
\nc{\epsahk}{\gapfn{0}}
\nc{\BH}{\mathbb H}
\nc{\BG}{\mathbb{G}}
\nc{\D}{\Delta}
\nc{\MF}{\mathcal{F}}
\nc{\One}[1]{\mathbbm{1}\{{#1}\}}
\nc{\bOne}{\mathbf{1}}
\nc{\Aopt}{\mathcal{A}^{\rm opt}}
\nc{\Amul}{\mathcal{A}^{\rm mul}}

\nc{\SQ}{\mathsf Q}

\nc{\DO}{\accentset{\circ}{\D}}
\nc{\mf}{\mathfrak}
\nc{\mfp}{\mathfrak{p}}
\nc{\mfq}{\mf{q}}
\nc{\mfx}{\mf{s}}
\nc{\Sp}{\mbox{Spec}}
\nc{\Spm}{\mbox{Specm}}
\nc{\hookuparrow}{\mathrel{\rotatebox[origin=c]{90}{$\hookrightarrow$}}}
\nc{\hookdownarrow}{\mathrel{\rotatebox[origin=c]{-90}{$\hookrightarrow$}}}
\nc{\hra}{\hookrightarrow}
\nc{\tra}{\twoheadrightarrow}
\nc{\sgn}{{\rm sgn}}
\nc{\aut}{{\rm Aut}}
\nc{\Hom}{{\rm Hom}}
\nc{\img}{{\rm Im}}
\DMO{\id}{Id}
\DMO{\supp}{supp}
\DMO{\KL}{KL}
\nc{\kld}[2]{D_{\mathsf{KL}}({#1}||{#2})}
\nc{\ren}[2]{D_2({#1}||{#2})}
\nc{\chisq}[2]{\chi^2({#1}||{#2})}
\nc{\tvd}[2]{D_{\mathsf{TV}}({#1}, {#2})}
\nc{\hell}[2]{D_{\mathsf{H}}^2({#1}, {#2})}
\nc{\dbi}[3][\pi]{D_{\mathsf{bi}}^{#1}({#2} \| {#3})}
\DMO{\BSS}{BSS}
\DMO{\BES}{BES}
\DMO{\BGS}{BGS}
\DMO{\poly}{poly}
\nc{\indep}{\perp}
\DMO{\sink}{sink}
\nc{\fp}[1]{\MP_1({#1})}
\nc{\BO}{\mathbb{O}}
\nc{\BT}{\mathbb{T}}

\nc{\RR}{\mathbb{R}}
\nc{\Gradient}{\nabla}
\DMO{\diag}{diag}
\DeclareMathOperator*{\EE}{\mathbb{E}}
\nc{\MQ}{\mathcal{Q}}
\nc{\ML}{\mathcal{L}}
\nc{\cPhi}{\bar \Phi}

\DeclareMathOperator*{\PR}{Pr}
\renewcommand{\Pr}{\PR}
\nc{\E}{\mathbb{E}}
\nc{\ra}{\rightarrow}
\renewcommand{\t}{\top}
\nc{\pmhc}[1]{\{-1,1\}^{#1}}
\nc{\Dbnd}{D}
\nc{\Bbnd}{B}

\nc{\Key}{\mathsf{KeyGen}}
\nc{\Enc}{\mathsf{Encode}}
\nc{\Encemb}{\mathsf{EncodeEmb}}
\nc{\Dec}{\mathsf{Decode}}
\nc{\sk}{\mathsf{sk}}
\nc{\pk}{\mathsf{pk}}
\nc{\lpk}{\ell_{\mathsf{pk}}}
\nc{\lsk}{\ell_{\mathsf{sk}}}
\nc{\msg}{\mathsf{m}}
\nc{\Adv}{\mathsf{Adv}}
\nc{\Red}{\mathsf{Red}}
\nc{\negl}{\mathsf{negl}}
\nc{\Ber}{\mathrm{Ber}}
\nc{\PRFPRC}{\mathsf{PRF\text{-}PRC}}
\nc{\wt}{\mathrm{wt}}
\nc{\res}[2]{{#1}_{#2}}
\nc{\bzero}{\mathbf{0}}
\nc{\Bin}{\mathrm{Bin}}
\nc{\Hyp}{\mathrm{Hyp}}

\nc{\Nrho}[1][\rho]{{N}_{#1}}
\nc{\Trho}[1][\rho]{\mathsf{T}_{#1}}
\nc{\hc}[1][n]{\{0,1\}^{#1}}
\nc{\Stab}{\mathbf{Stab}}
\nc{\bW}{\mathbf{W}}
\nc{\NS}{{\mathbf{NS}}}

\nc{\KeyS}{\mathsf{KeyGen_{Sub}}}
\nc{\EncS}{\mathsf{Encode_{Sub}}}
\nc{\DecS}{\mathsf{Decode_{Sub}}}
\nc{\WeightPerturb}{\mathsf{WeightPerturb}}
\nc{\Unique}{\mathsf{Unique}}
\nc{\PRCS}{\mathsf{PRC_{Sub}}}
\nc{\PRC}{\mathsf{PRC}}
\nc{\PRCI}{\mathsf{PRC_{Idx}}}
\nc{\SampleUnique}{\mathsf{SampleUnique}}
\nc{\PerturbDifference}{\mathsf{PerturbDifference}}

\nc{\Model}{\mathsf{Model}}
\nc{\Modelo}{\overline{\Model}}
\nc{\prompt}{\mathtt{PROMPT}}
\nc{\Setup}{\mathsf{Setup}}
\nc{\Detect}{\mathsf{Detect}}
\nc{\Sigprc}{\Sigma_{\mathsf{PRC}}}
\nc{\Wat}{\mathsf{Wat}}
\nc{\term}{\mathtt{END}}
\nc{\tok}{\mathsf{t}}
\nc{\True}{\textsf{True}}
\nc{\False}{\textsf{False}}
\nc{\Eemb}{\ME_{\mathsf{Emb}}}
\nc{\hist}{\mathsf{hist}}
\nc{\hh}{\mathsf{h}}
\nc{\freq}{\mathsf{freq}}
\nc{\ff}{\mathsf{f}}

\nc{\Hemp}[1]{H_{\mathsf{e}}^{#1}}
\nc{\Hempt}[1]{\bar{H}_{\mathsf{e}}^{#1}}
\nc{\Hemptil}[1]{\tilde{H}_{\mathsf{e}}^{#1}}
\nc{\Spread}[1]{S^{#1}}
\nc{\Hmean}[1]{H_{\mathsf{m}}^{#1}}
\nc{\partition}[1][n,q]{P^{\mathsf{ptn}}_{#1}}
\nc{\Crob}{C_{\mathsf{rob}}}
\nc{\Lmax}{L_{\mathsf{max}}}
\nc{\skwat}{\sk_{\mathsf{Wat}}}
\nc{\EmbedToken}{\mathsf{EmbedChar}}
\nc{\len}{\mathrm{len}}
\nc{\Esub}{\ME_{\mathsf{sub}}}
\nc{\Ecomp}{\ME_{\mathsf{comp}}}
\nc{\comp}{\mathsf{c}}
\nc{\SE}{\mathscr{E}}
\nc{\alphb}{q}
\nc{\tAdv}{\widetilde{\Adv}}
\nc{\Funif}{{F_{\mathsf{Unif}}}}
\nc{\Alg}{\mathsf{Alg}}
\nc{\Majority}{\mathsf{Maj}}
\nc{\Dist}{\mathsf{Dist}}
\nc{\edit}{edit\xspace}
\nc{\Edit}{Edit\xspace}
\nc{\Wcomp}{\MW^{\mathsf{comp}}}

\nc{\INS}{\mathsf{INS}}
\nc{\CNS}{\mathsf{CNS}}
\nc{\cdist}{\stackrel{\mathrm{c}}{\sim}}
\nc{\SU}{\mathscr{U}}
\nc{\rr}{\bar{n}}

\nc{\KeyGen}{\mathsf{KeyGen}}
\nc{\ED}{D_{\mathsf{ED}}}
\nc{\Ham}{D_{\mathsf{Ham}}}
\nc{\bin}{\mathsf{bin}}
\nc{\EDball}{\mathcal{B}_{\mathsf{ED}}}
\nc{\SEDball}{\mathcal{B}_{\mathsf{Ham,ED}}}
\nc{\LEDball}{\mathcal{B}_{\mathsf{len,ED}}}
\nc{\epED}{\varepsilon_{\mathsf{ED}}}
\nc{\epDec}{\varepsilon_{\mathsf{Dec}}}
\nc{\Eedit}{\mathscr{E}^{\mathsf{edit}}}
\nc{\Egood}{\ME_{\mathsf{good}}}
\nc{\PermEnc}{\mathsf{PermEncode}}
\nc{\dham}{d_{\mathsf{H}}}
\nc{\dedit}{d_{\mathsf{E}}}
\nc{\pDec}{p_{\mathsf{Dec}}}
\nc{\Rclean}{R_{\mathsf{Clean}}}
\nc{\adv}{\mathcal{A}}
\nc{\bi}{\mathbf{i}}

\nc{\bj}{\mathbf{j}}
\nc{\bp}{\mathbf{p}}
\nc{\Ologit}{\MO_{\mathsf{logit}}}
\nc{\Osamp}{\MO_{\mathsf{samp}}}
\nc{\epapx}{\varepsilon_{\mathsf{apx}}}
\nc{\DistSpanner}{\textsf{DistSpanner}\xspace}
\nc{\epbase}{\vep_{\mathsf{base}}}
\nc{\cnorm}{\beta}
\nc{\Lapx}{\logit_{\mathsf{oracle}}}
\nc{\BB}{\mathbb{B}}
\nc{\Bfut}{\mathbb{A}^{\mathsf{fut}}}
\nc{\Bhist}{\mathbb{A}^{\mathsf{hist}}}
\nc{\gamthres}{\gamma_{\mathsf{thres}}}
\nc{\softmax}{\mathrm{softmax}}
\nc{\Espanner}{\ME_{\mathrm{spanner}}}
\nc{\Eoracle}{\ME_{\mathrm{oracle}}}
\nc{\epavg}{\vep_{\mathsf{avg}}}
\nc{\tilOlogit}{\tilde{\MO}_{\mathsf{logit}}}
\nc{\ds}{{\bar{d}}}
\nc{\Lmat}{\logitmatrix_{\model}}
\nc{\feasSet}{\mathcal{G}}
\nc{\projOne}{\Pi_{\mathbf{1}}}

\newcommand{\logit}{\mathbb{L}}
\newcommand{\model}{\mathbb{M}}
\newcommand{\logitmatrix}{\ML}

\title{Provably Learning from Modern Language Models via Low Logit Rank}

\author{
Noah Golowich \\
Microsoft Research \\
\texttt{noah.golowich@austin.utexas.edu}
\and
Allen Liu \\
UC Berkeley \\
\texttt{aliu42@berkeley.edu}
\and
Abhishek Shetty \\
MIT \\
\texttt{shetty@mit.edu}}

\begin{document}
    \maketitle

\begin{abstract}
While modern language models and their inner workings are incredibly complex, recent work \cite{golowich2025sequenceslogitsreveallow} has proposed a simple and potentially tractable abstraction for them through the observation that empirically, these language models all seem to have approximately \emph{low logit rank}. 
Roughly, this means that a matrix formed by the model's log probabilities of various tokens conditioned on certain sequences of tokens is well approximated by a low rank matrix.

In this paper, our focus is on understanding how this structure can be exploited algorithmically for obtaining provable learning guarantees.  
Since low logit rank models can encode hard-to-learn distributions such as noisy parities, we study a query learning model with logit queries that reflects the access model for common APIs. 
Our main result is an \emph{efficient algorithm for learning any approximately low logit rank model  from queries}. We emphasize that our structural assumption closely reflects the behavior that is empirically observed in modern language models. 
Thus, our result gives what we believe is the first end-to-end learning guarantee for a generative model that plausibly captures modern language models.

\end{abstract}

\thispagestyle{empty}

\newpage
\thispagestyle{empty}
\tableofcontents
\newpage

\setcounter{page}{1}

\section{Introduction} \label{sec:intro} 
  
In recent years, \emph{large language models} (LLMs) have demonstrated remarkable success in a wide range of areas including natural language processing, computational biology, and mathematics, amongst others. Due to the complexity inherent to the neural network architectures used in these models, most notably \emph{transformers}, %
a major challenge is to obtain a theoretical understanding of how and why these models work, as well as why some of their failure modes arise. Accordingly, a number of theoretical works have attempted to introduce and analyze \emph{simple abstractions} that aim to capture some of the core features of LLMs: for instance, there have been several papers studying the problem of learning a single attention layer \cite{chen2025provably,geshkovski2024dynamic, geshkovski2023emergence}, analyzing the expressivity of LLMs \cite{merrill2025theory}, examining the ``limiting problem'' of language generation \cite{kleinberg2024language}, and exploring many other directions.

Despite the above works, our theoretical understanding of LLMs remains in its nascent stages. Indeed, essentially all of the research in this area 
makes a number of assumptions which do not closely resemble properties of modern LLMs, including particular architectural assumptions \cite{chen2025provably,geshkovski2024dynamic, geshkovski2023emergence} or strong restrictions on the data distribution, such as to specific synthetic tasks \cite{huang2025transformers,wang2025learning}. 
Moreover, our lack of theoretical understanding has been a 
fundamental impediment towards both %
the principled design of better models and the development of mechanisms to ensure their safe deployment in real-world applications. %

In this work, we aim to bridge this gap by taking a different approach: 
as opposed to directly trying to model the particularities of modern architectures or common benchmarks, we instead view language models simply as probability distributions over sequences of tokens and aim to understand the core structural properties that they exhibit empirically. In particular, we ask the following questions:
\begin{itemize}
    \item What structural properties of the distributions induced by common language models can we identify?
    \item Are these properties sufficient to allow theoretical analysis and insights? In particular, can we derive efficient end-to-end learning guarantees under such properties?
\end{itemize}

One promising structural property that has been recently proposed is that of \emph{low logit rank} \cite{golowich2025sequenceslogitsreveallow}. 
This property posits that a certain matrix defined from the language model is approximated by a low rank matrix. Roughly speaking, the entries of this matrix are given by the log probabilities of tokens in certain sequences (e.g., responses) conditioned on other sequences (e.g., prompts).
Surprisingly, \cite{golowich2025sequenceslogitsreveallow} showed that this property is (approximately) exhibited across the board by a wide range of modern language models. Furthermore, this property can be used to demonstrate surprising empirical phenomena such as the ability to generate from a language model conditioned on certain prompts by only querying the model on \emph{unrelated} prompts. %

In this paper, we advance this line of inquiry by establishing a formal learning guarantee for language models that exhibit approximately low logit rank.
To the best of our knowledge, this is:
\begin{center}
   \emph{\textbf{The first end-to-end learning guarantee for distributions that plausibly capture modern language models.}}
\end{center}
In fact, running a simplified version of our algorithm (\cref{alg:learning-approx}) on {Tinystories-8M}, we learn a very simple low logit rank model that generates readable stories such as \footnote{These were generated using a model with rank $4000$, with prompt ``Once upon a time".}: 
\begin{quote}
\emph{``Once upon a time, there was a baby bear. He was very upset, but he still soundly looked down the newspaper" }
\end{quote}
\begin{quote}
\emph{``Once upon a time, there was a big house, where a family lived in the house. One day was a little girl"}
\end{quote}

Whereas previous works studying provable learning guarantees for LLMs all worked with models such as Hidden Markov models \cite{mahajan2023learning, liu2025model} or single attention layers \cite{chen2025provably} which are too simplified to generate meaningful language, we believe substantial empirical evidence (\cite{golowich2025sequenceslogitsreveallow}, and Figure~\ref{fig:approx}) supports the notion of low logit rank as a much more realistic modeling assumption. More generally, we hope that our results will spur further research strengthening the connections between learning theory and modern language models.

\subsection{Low Logit Rank}

In this section, we formalize the notion of low logit rank and discuss the empirical observation that modern language models exhibit this property.
Towards this, we begin with an abstract view of language models as distributions over sequences of tokens from some finite set $\Sigma$, which we call the \emph{alphabet}, or \emph{token space}. For simplicity, we will consider a fixed sequence length $T$. We let $\Sigma^\st$ denote the set of all finite sequences of tokens from $\Sigma$, $\Sigma^t$ denote the set of sequences of length $t$, and $\Sigma^{\leq t}$ denote the set of sequences of length at most $t$. For $y_1, \ldots, y_t \in \Sigma$, we let $y_{1:t} := (y_1, \ldots, y_t)$. For sequences $h,f \in \Sigma^\st$, we let $h \circ f$ denote their concatenation. 

\begin{definition}[Language Model]
    Let $T$ be a positive integer denoting the sequence length. 
    A \emph{language model}, $\model$, is a distribution over sequences of tokens of length $T$ from $\Sigma$, i.e., $\model \in \Delta(\Sigma^T)$. 
\end{definition}

As language models are commonly considered autoregressively, we will set up notation to reflect this.
For a sequence $y_{1:t} = (y_1, \ldots, y_t) \in \Sigma^t$ and a token $y \in \Sigma$, we will denote by 
\begin{align}
  \model(y \mid y_{1:t}) := \Pr_{\model} (y_{t+1} = y \mid y_{1:t}). \nonumber
\end{align}
More generally, for sequences $h \in \Sigma^t$, $f \in \Sigma^s$ with $t + s \leq T$, we will write
\begin{align}
  \model(f \mid h) := \Pr_{\model} \left(y_{t+1:t+s} = f \mid y_{1:t} = h\right)\nonumber
\end{align}
and let $f \sim \model (\cdot | h) $ mean that $f$ is sampled from the distribution $\model(\cdot | h)$ while letting the length of $f$ be given by context. In this context, we will often refer to $h$ as a \emph{history} and to $f$ as a \emph{future}.

An interesting feature of modern language models is that they all first compute a vector of \emph{logits} $\logit \in \R^{|\Sigma|}$ and then sample the next token according to the distribution $\text{softmax}(\logit)$.\footnote{Recall that the softmax function maps a vector $\logit = (\logit(y))_{y \in \Sigma}$ to the probability distribution $P \in \Delta(\Sigma)$ defined by $P(y) \propto\exp(\logit(y))$.} The logits will play a crucial role \---- as we will see, for modern language models, there is significantly more structure in the logits compared to the raw probabilities.

\begin{definition}[(Mean-centered) Logits]
Given a sequence $y_{1:t} \in \Sigma^t$ and $y_{t+1} \in \Sigma$, we define the \emph{mean-centered logits} as $\logit_{\model}(y_{t+1}| y_{1:t}) = \log \model (y_{t+1}| y_{1:t}) - \frac{1}{|\Sigma|} \sum_{y \in \Sigma} \log \model (y| y_{1:t}) $.
To simplify notation at times we will write $\logit_\model(y_{1:t+1}) := \logit_\model(y_{t+1} \mid y_{1:t})$. 
\end{definition}

Now we introduce the main structural property we study, \emph{low logit rank}, which governs the rank of a collection of matrices associated to any language model $\model$ we refer to as \emph{logit matrices}.

\begin{definition}[Logit Matrix]
  \label{def:extended-logit-matrix}
  Fix a model $\model \in \Delta(\Sigma^T)$.  Given subsets of sequences $\mathcal{H} \subset \Sigma^*$, $\mathcal{F} \subset \Sigma^*$ (which we refer to as \emph{histories} and \emph{futures}, respectively) so that $|h| + |f| \leq T-1$ for any $h \in \MH, f \in \MF$, the associated \emph{logit matrix} $\logitmatrix_{\model}(\MH, \MF) \in \BR^{\MH \times (\MF \times \Sigma)}$ is defined as follows: for $h \in \MH$ and $(f,z) \in \MF \times \Sigma$, the $(h, (f,z))$-entry is given by
  \begin{align}
 \logitmatrix_{\model}(\MH, \MF)_{(h, (f,z))} := \logit_{\model}(z \mid h \circ f)\nonumber.
    \end{align}
\end{definition}

\begin{definition}[(Exact) Low Logit Rank]
  \label{def:logit-rank-exact}
A model $\model$ is said to have \emph{logit rank $d$} if for all choices of $\MH, \MF$, the matrix $\logitmatrix_{\model}(\MH, \MF)$ has rank at most $d$. 
\end{definition}

We now briefly discuss how to interpret this definition. First note that when $\calF$ contains all possible sequences (of up to some length $t$), then a single row corresponding to $h$, $\logitmatrix_{\model}(h, \calF)$, contains all of the information necessary to sample the next $t$ tokens conditioned on the history $h$. %
When the extended logit matrix has low rank, it means that there is a basis of histories $h_1, \dots , h_d$ such that any row $\logitmatrix_{\model}(h, \calF)$ can be written as a linear combination of the rows for these basis histories. This then implies that for each history $h$, there is a $d$-dimensional hidden state such that the information needed to sample the tokens after $h$ is a linear function of this hidden state (in logit space). This intuition was formalized in \cite{golowich2025sequenceslogitsreveallow}, which shows that a low logit rank language model is essentially equivalent to a simple latent variable model called an Input-Switched Affine Network (ISAN).

\begin{definition}[(Time-varying) ISAN]
  \label{def:isan}
A time-varying ISAN of sequence length $T$ is specified by matrices $\BA_{y,t} \in \R^{d \times d}, \BB_t \in \R^{|\Sigma| \times d}$ for $z \in \Sigma, t \in [T]$ and an initial state $ \mu \in \R^d$.  It defines a distribution over sequences of tokens $(y_1,z_2, \dots , y_T) \in \Sigma^T$ as follows: $y_{t}$ is sampled from $\softmax(\BB_{t} x_{t-1})$, where $x_0 = \mu$.  Then, the hidden state is updated as $x_{t} = \BA_{y_{t},t} x_{t-1}$.
\end{definition}

\begin{fact}[Equivalence from \cite{golowich2025sequenceslogitsreveallow}]
Let $\model$ be a language model over sequences of length $T$.  Then $\model$ is expressible as a time-varying ISAN with hidden dimension $d$ if and only if the logit matrix $\logitmatrix_{\model}(\Sigma^{t}, \Sigma^{\leq T - t-1})$ has rank at most $d$ for all $t \leq T$.
\end{fact}

The key observation of \cite{golowich2025sequenceslogitsreveallow} is that many modern language models \emph{approximately} exhibit low logit rank; further, \cite{golowich2025sequenceslogitsreveallow} established many surprising implications of this fact. 
Thus, in order to reason about modern language models, we need to study \emph{approximately} low logit rank models. 
With this motivation, we ask the following questions:

\begin{center}
\textbf{Question 1:} What is the appropriate notion of approximation that reflects the properties of real language models?
\end{center}

\begin{center}
\textbf{Question 2:} For a language model that has approximate low logit rank in the appropriate sense, can we efficiently learn it? %
\end{center}

\subsubsection{Approximate Low Logit Rank}
Towards answering the first question above, we introduce a notion of approximate low logit rank which is supported by experiments on modern LLMs. For integers $0 \leq s < t \leq T$, we let $\model_{s:t}$ denote the marginal distribution of $y_{s:t}$ for $ y  \sim \model$.  
In \cref{def:avg-closeness-weak}, we define a language model to be approximately low logit rank if, roughly speaking, its \emph{logit matrices $\logitmatrix_{\model}(\Sigma^s, \Sigma^{\leq T-s-1})$ are entrywise-close to those of a low-rank matrix on average when the row $h$ and column $(f,y)$ are drawn independently according to their marginals under the model $\model$}.\footnote{For technical reasons, we also require that the last token of the row and column is uniformly random, as in \cref{eq:avg-closeness-weak}. Moreover, we will need a slightly stronger notion of ``low-rank'', namely that there is a low-dimensional factorization with bounded norms, formalized in \cref{def:bounded-logit}.} %
\begin{definition}
\label{def:avg-closeness-weak}
Fix $\ep, \alpha > 0$ and a distribution $\model \in \Delta(\Sigma^{T})$. We say that $\model$ has \emph{$\ep$-approximate logit rank $d$} if for each $0 \leq s \leq T$ there is a matrix $\logitmatrix\^s \in \BR^{\Sigma^{ s} \times (\Sigma^{\leq T-s})}$ of rank $d$ so that for each $0 \leq s < t \leq T$, we have
\begin{align}
\EE_{\substack{h \sim \model_{1:s-1} \times \Unif(\Sigma), \\ f \sim \model_{s+1:t-1},\ y_t \sim \Unif(\Sigma)}} \left[ \left| \logitmatrix\^s(h, (f,y_t)) - \logit_\model(y_t \mid h \circ f) \right| \right] \leq \ep\label{eq:avg-closeness-weak}.
\end{align}
Above the notation $h \sim \model_{1:s-1} \times \Unif(\Sigma)$ means that we draw $y_{1:s-1} \sim \model_{1:s-1}$ and $y_s \sim \Unif(\Sigma)$ independently and set $h = y_{1:s}$; and $\logitmatrix\^s(h, (f,y_t))$ denotes the $(h, (f,y_t))$-entry.
\end{definition}

\paragraph{Empirical Validation.}
Next, we perform an experiment, paralleling those in \cite{golowich2025sequenceslogitsreveallow}, to measure the degree to which the logit matrices $\logitmatrix_{\model}(\calH, \calF)$ for modern language models are approximately low logit rank in the sense of \cref{def:avg-closeness-weak}. In particular, we study the {OLMo2-1b} language model  $\model$ \cite{olmo2}. We consider the logit matrix defined by sets $\calH, \calF$ of $n = 4 \cdot 10^3$ histories and futures of lengths between $1$ and $20$ tokens.  For each length, the histories and futures are generated from the language model according to the distributions in \Cref{def:avg-closeness-weak}. \noah{discuss any differences, e.g., 20 tokens per future (maybe in appendix)} %

\begin{figure}[!h]  
\begin{center}
\includegraphics[width = 0.7\linewidth]{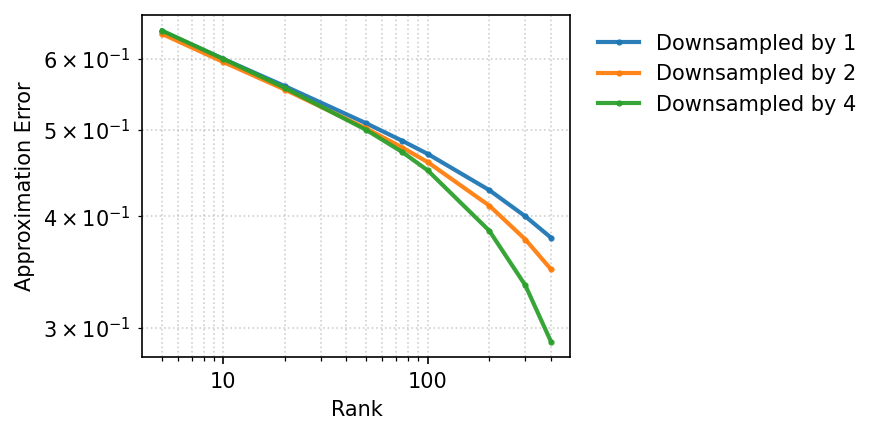}
\caption{Low-rank approximation error (measured by \emph{average $L^1$ error}) of the extended logit matrix for {OLMo2-1b}. For fixed sets $\MH, \MF$, the approximation errors for the logit matrix $\logitmatrix_\model(\MH, \MF)$ behave according to a similar power law as to those of various sub-matrices with $2$ or $4$ times fewer entries.} 
\label{fig:approx}
\end{center}
\end{figure}

We then construct the matrix $\logitmatrix_{\model}(\calH, \calF)$ and measure its approximability by rank-$d$ matrices, for various values of $d$. Observe that for any approximation, say $\logitmatrix' \in \BR^{\MH \times (\MF \times \Sigma)}$, we can compute the average entrywise $L^1$ error 
\begin{align}
\label{eq:avg-closeness-emp}
\frac{1}{|\MH| |\MF| |\Sigma|} \sum_{h,f,y} |\logitmatrix_{\model}(h,(f,y)) - \logitmatrix'(h,(f,y))|,
\end{align}
which should be interpreted as an empirical analogue of \cref{eq:avg-closeness-weak}. Note that, since $\Sigma^s, \Sigma^{\leq T-s}$ are exponentially large, it is infeasible to directly show that something like \cref{eq:avg-closeness-weak} holds directly; however, we can obtain convincing evidence for it by observing how the empirical quantity in \cref{eq:avg-closeness-emp} changes as we \emph{scale up} $\MH, \MF$ generated as above.

Accordingly, in \cref{fig:approx}, we plot the average entrywise $L^1$ error between the true logit matrix for $\model$ and the rank-$d$ approximations for various ranks $d$ between $5$ and $500$. 
We make two important observations about the plot:
\begin{enumerate}
\item The error is essentially the same even when we subsample $\calH, \calF$ down significantly
\item The error follows a power law with slope $ \approx -0.1$
\end{enumerate}
The first suggests that the approximation error would remain low even if measured across  the entire logit matrices (i.e., if we took $\MH = \Sigma^s, \MF = \Sigma^{\leq T-s-1}$) and weighted histories and futures according to $\model$, as in \cref{def:avg-closeness-weak}. %
The second, if we extrapolate the power law, implies that if we take rank $d \geq \poly(1/\ep)$, then the rank-$d$ approximation would have average $L^1$ error at most $\ep$. These observations give credence to the notion of approximation in \Cref{def:avg-closeness-weak} as a clean mathematical notion that is faithful to real language models, especially compared to other a priori natural notions of approximation such as $L^{\infty}$ closeness.

\subsection{Our Contributions} \label{sec:contributions}

The main contribution of this paper is to give an end-to-end learning guarantee for language models that have approximately low logit rank in the sense described above. 

The most basic and well-studied setting of ``learning'' a distribution $\model$ is to learn using \emph{i.i.d.~samples} from $\model$. Alas, this task is computationally intractable even for very simple generative models \---- \cite{golowich2025sequenceslogitsreveallow} shows that even models of logit rank $d=2$ can express the notorious \emph{noisy parities problem} \cite{blum2003noise}, which is widely believed to be computationally intractable to learn. %

To circumvent this barrier, we opt to strengthen the power of the learning algorithm by allowing it to make \emph{queries} to the target distribution $\model$. In general, query learning models have an extensive history as a means of understanding and circumventing computational hardness barriers \cite{kushilevitz1991learning,angluin1987learning}, and have garnered increased interest in recent years as a means to study the power of query access to machine learning model APIs for applications such as \emph{model stealing} and \emph{distillation} \cite{liu2025model,mahajan2023learning}. Indeed, \emph{distillation}, which refers to the process of using query access to a teacher model to train a (typically smaller) student model, has become a popular approach to train language models \cite{sanh2020distillbert,agarwal2024onpolicy}; and notable works \cite{carlini2024stealing} have shown that by making appropriate queries one can \emph{steal} the last layer of a closed-source production model. For autoregressive language models, a common query format (which is supported by many APIs) is to allow the user to query an arbitrary sequence $y_{1:t}$, and to respond with the logits for the next token, i.e., $\logit_\model(y_{t+1} = \cdot \mid y_{1:t})$.  
This \emph{logit query} setting is exactly the one that we study: %
\begin{definition}[Logit oracle]
  \label{def:exact-logit-oracle}
An \emph{(exact) logit oracle} $\Ologit$ takes as input a sequence $y_{1:t} \in\Sigma^t$ and returns the vector $\log \model(y_{t+1} = \cdot \mid y_{1:t}) \in \BR^\Sigma$, for any $y_{1:t} \in \Sigma^t$.
\end{definition}
Our main result shows that \emph{any approximately low logit rank language model} can be efficiently learned using queries to a logit oracle with time and query complexity that are polynomial in all relevant parameters: \noah{if time -- resemblance to some rl strategies where you merge trajs}

\begin{theorem}[Main theorem; Informal version of \cref{thm:approx-main}]
  \label{thm:approx-main-informal}
  There is a polynomial-time algorithm that, given the ability to make logit queries to a distribution $\model$ with $\epavg$-approximate logit rank $d$, outputs an efficiently sampleable distribution $\hat \model$ satisfying $\tvd{\hat \model}{\model} \leq \ep^\st$ for $ \ep^{\st} \approx \poly(\epavg, d, |\Sigma|, T)$.
\end{theorem}

We emphasize again that the assumptions we need for learning are exactly the same as empirically observed in \cref{fig:approx} (see also \cite{golowich2025sequenceslogitsreveallow})
up to the exponent in the relationship between $\ep$ and $d$ \---- we need that a given rank $d$ gives an $\ep$-approximation for $\ep = 1/d^{C^\st}$ for a sufficiently large constant $C^\st$ (while \cref{fig:approx} shows a decay of approximately $ 1/d^{0.1}$ in practice).
Nevertheless, we believe our results give the first end-to-end provable learning guarantees for a model that can serve as a reasonable approximation for modern language models. %

In \cref{sec:km}, we discuss how the guarantee of \cref{thm:approx-main-informal} yields a weak form of the celebrated \emph{Kushilevitz-Mansour} theorem \cite{kushilevitz1991learning} for learning sparse Boolean functions with queries, with respect to the uniform distribution. We emphasize that our setting goes much beyond that setting, as our model can capture complex distributions over sequences which are far from uniform (and in which the Fourier-analytic tools of \cite{kushilevitz1991learning} seem unlikely to apply).

\paragraph{Going forward: beyond the i.i.d.~model.} We remark that more broadly, \emph{modern LLM training pipelines are increasingly deviating from the traditional paradigm of learning from i.i.d.~data}. In addition to the settings of model stealing and distillation discussed above, this trend has emerged in various ways: for instance, LLMs are often trained using data that is generated in an online manner through \emph{reinforcement learning} \cite{christiano2017deep}, or using data fed to the algorithm in a certain order so as to resemble a \emph{curriculum} \cite{lee2025selfimproving}. %
These empirical advances call for a more thorough and principled treatment of the various ways a learning algorithm can interact with its ``target distribution'', beyond the traditional i.i.d.~setting. We view our result as a  step in this direction, as well as towards the larger goal of strengthening the connection between learning theory and practical language modeling research. Accordingly, in \cref{sec:conclusions} we discuss some directions for future work.

\section{Technical Overview} \label{sec:techover}

In this section, we overview the main technical ideas behind our algorithm and its analysis. The proof of \cref{thm:approx-main-informal} proceeds in two parts: first, in \cref{sec:exact-model-approx-oracle,sec:ep-dimred,sec:lp-coefficients,sec:alg-overview-intro}, we consider the problem of learning a model $\model$ with a logit oracle $\Ologit$ which returns answers that are $\epapx$-close to the true logits of $\model$. For the purposes of this technical overview, we assume further that $\model$ is of exact logit rank $d$.
In the full proof (in \cref{sec:analysis}), we will have to actually make the slightly modified assumption that \emph{the logits which $\Ologit$ returns over the course of the algorithm are consistent with low-rank matrices} (see \cref{def:bounded-logit-oracle} for a formal definition); however, the main ideas can be communicated without this additional complication. %
In \cref{sec:approx-model}, we discuss how to extend the preceding argument to prove our main result, which only assumes that the logit matrices of $\model$ are approximated (in the average sense of \cref{def:avg-closeness-weak}) by low-rank matrices. %

\subsection{Learning with an approximate logit oracle}
\label{sec:exact-model-approx-oracle}
Let us suppose we are given a model $\model \in \Delta(\Sigma^T)$ with logit rank $d$, and that we are given access to an \emph{$\epapx$-approximate logit oracle $\Ologit$}, which when queried on a sequence $y_{1:t} \in \Sigma^t$, returns $\Ologit(y_{1:t}) \in \BR^\Sigma$ satisfying $| \Ologit(y_{1:t})_{y_{t+1}} - \log \model(y_{t+1} \mid y_{1:t}) | \leq \epapx$ for all $y_{t+1} \in \Sigma$ (see \cref{def:approx-logit-oracle}). Our goal is to learn a language model $\hat \model$ (equipped with an efficient sampling algorithm) that is close to $\model$ in total variation distance. Note that, while $\model$ can be expressed as a rank-$d$ ISAN (\cref{def:isan}), our learning algorithm will be \emph{improper} in the sense that the distribution $\hat \model$ may not be a rank-$d$ ISAN. %

To represent the distribution $\hat \model$, we will exploit the fact that the logit matrices of $\model$ have rank at most $d$. Recall the definition of the logit matrix $\Lmat(\MH, \MF)$ corresponding to sets $\MH \subset \Sigma^\st, \MF \subset \Sigma^\st$ in \cref{def:extended-logit-matrix}. For $t \in [T]$, we will, thoughout this section, consider the logit matrix $\Lmat(\Sigma^{t-1}, \Sigma^{\leq T-t})$. Accordingly, for sequences $h \in \Sigma^{t-1}, f \in \Sigma^{\leq T-t}$ and $y \in \Sigma$, we denote the $(h,(f,y))$ entry of this logit matrix by $\Lmat(h, (f,y)) := \logit(y \mid h \circ f)$; we will often shorten this notation by writing $\Lmat(h,f') := \Lmat(h, (f,y))$, for $f' := f \circ y$. We will also denote the row of the logit matrix corresponding to $h$ by $\Lmat(h, \cdot): = (\Lmat(h, f'))_{f' \in \Sigma^{\leq T-t} \times \Sigma}$.

Given the low-rank structure of $\Lmat$, a natural way to specify the distribution $\hat\model$ (which parallels that taken in prior works on learning \emph{low-rank} distributions \cite{liu2025model,mahajan2023learning}) is via the following data:
\begin{itemize}
\item For each $t \in [T]$, a collection of ``basis histories" $h_{t,1}, \ldots, h_{t,\ds} \in \Sigma^{t-1}$, where $\ds \in \BN$ is a parameter (which will in general be larger than the logit rank $d$).

\item For each $t \in [t]$, $i \in [\ds]$, and $y_t \in \Sigma$, a way of expressing the the row $\Lmat(h_{t,i} \circ y_t, \cdot)$ of the logit matrix as a linear combination of the ``basis rows'' at step $t+1$: in particular, this data consists of vectors $c_{t,i,y_t} \in \BR^{\ds}$, for each $i \in [\ds], y_t \in \Sigma$, satisfying
\begin{align}
\Lmat(h_{t,i} \circ y_t, \cdot) \approx \sum_{j=1}^\ds c_{t,i,y_t,j} \cdot \Lmat(h_{t+1,j}, \cdot) \quad \forall i \in [\ds],\ y_t \in \Sigma.\label{eq:lmat-basis-change}
\end{align}
\end{itemize}
To start, we shall assume that the histories $h_{t,i}$ and coefficients $c_{t,i,y_t,j}$ as above are given to us; in the course of the below arguments, we shall explain how the algorithm can actually learn them.

Given data as described above, there is a natural way to sample from a distribution $\hat\model$ approximating $\model$: having already sampled tokens $y_{1:t-1}$, we inductively maintain coefficients $\hat c_{t,1}, \ldots, \hat c_{t,\ds} \in \BR$ for which

\begin{align}
\Lmat(y_{1:t-1}, \cdot) \approx \sum_{i=1}^\ds \hat c_{t,i} \cdot \Lmat(h_{t,i}, \cdot).\label{eq:lmat-induction}
\end{align}
Then defining $\hat \ell_t \in \BR^\Sigma$ by $\hat \ell_t(y_t) := \sum_{i=1}^{\ds} \hat c_{t,i} \cdot \Lmat(h_{t,i}, (\emptyset, y_t))$ (which contains the entries of the vector $\sum_{i=1}^{\ds} \hat c_{t,i} \cdot \Lmat(h_{t,i}, \cdot)$ corresponding to the empty future $\emptyset$), we sample $y_t \sim \softmax(\hat \ell_t)$. The approximate equality \cref{eq:lmat-induction} ensures that the distribution of $y_t$ is close to that of $\model(y_t = \cdot \mid y_{1:t-1})$. Finally, we update the coefficients $\hat c_{t+1,1}, \ldots, \hat c_{t+1,\ds} \in \BR$ by setting
\begin{align}
\hat c_{t+1,j} := \sum_{i=1}^\ds \hat c_{t,i} \cdot c_{t,i,y_t,j} \quad \forall j \in [\ds],\label{eq:hatc-update}
\end{align}
which ensures that
\begin{align}
\sum_{j=1}^\ds \hat c_{t+1,j} \cdot \Lmat(h_{t+1,j}, \cdot) = \sum_{j=1}^\ds \sum_{i=1}^\ds \hat c_{t,i} \cdot c_{t,i,y_t,j} \cdot \Lmat(h_{t+1,j}, \cdot)  & \nonumber\\
\approx \sum_{i=1}^\ds \hat c_{t,i} \cdot \Lmat(h_{t,i} \circ y_t, \cdot) \approx \Lmat(y_{1:t}, \cdot),\label{eq:complete-induction}
\end{align}
where the first approximate equality uses the ``basis change'' property of \cref{eq:lmat-basis-change} and the second approximate equality uses the inductive hypothesis of \cref{eq:lmat-induction}. Moreover, in the final approximate equality we use the fact that $\Lmat(h_{t,i} \circ y_t, (f,y)) = \Lmat(h_{t,i}, (y_t \circ f, y))$ (which allows us to apply \cref{eq:lmat-induction}). Thus \cref{eq:complete-induction} ensures that the inductive hypothesis holds at step $t$. 

Unfortunately, the above argument suffers from a few significant gaps, which necessitate substantial modifications:
\begin{enumerate}
  \item First, it is not at all clear how we can ensure \cref{eq:lmat-basis-change} holds -- this (approximate) equality involves vectors with exponentially many entries (indexed by sequences of tokens of length up to $T-t$), so it is of course infeasible to even compute these vectors.
\item Even if we can resolve the first issue, the definition of the vector $\hat c_{t+1} \in \BR^{\ds}$ in terms of $\hat c_t \in \BR^{\ds}$ in \cref{eq:hatc-update} will lead, in general, to the size of the entries of these vectors growing exponentially in $t$. This results from the fact that the coefficients $c_{t,i,y_t, j}$ will be of absolute value $\Theta(1)$, meaning that the best bound we can establish is of the form $\| \hat c_{t+1} \|_\infty \leq O(\ds) \cdot \| \hat c_t \|_\infty$, and thus $\| \hat c_t \|_\infty \leq O(\ds)^t$ for $t \in [T]$. 
\end{enumerate}
In the subsequent sections, we discuss how to resolve these issues: we discuss the first issue above in \cref{sec:ep-dimred} and the second issue in \cref{sec:lp-coefficients}. While prior work on learning low rank distributions \cite{mahajan2023learning,liu2025model} has faced obstacles analogous to the above ones, as we discuss below existing techniques fail in the setting that $\model$ is of low \emph{logit} rank and thus new ideas are required. 

\subsection{Approximating rows of the logit matrix via the elliptical potential lemma}
\label{sec:ep-dimred}
In this section, we address the first issue mentioned above, namely the fact that the vectors $\Lmat(h_{t,i} \circ y_t, \cdot)$ and $\Lmat(h_{t+1,j}, \cdot)$ in \cref{eq:lmat-basis-change} have exponentially many entries. A very natural way to attempt to deal with this issue is as follows: given a set of histories $h_1, \ldots, h_n \in \Sigma^{t-1}$ (e.g., the histories $h_{t-1,i} \circ y_{t-1},\ h_{t,j}$ in the above context), we approximate the vectors $\Lmat(h_i, \cdot)$ by sampling $m$ futures $f$ from $\model$ and letting $L_i \in \BR^m$ denote the vector whose entries are $\Lmat(h_i, f)$, where $f$ ranges over the $m$ sampled futures. %
Indeed, this is the approach taken by previous works studying the low-rank problem \cite{liu2025model,mahajan2023learning}.

In what sense should the vectors $L_i \in \BR^{m}$ ``approximate'' the rows $\Lmat(h_i, \cdot)$ for $i \in [n]$? Ideally, we would want that if (and only if) any linear combination between these rows holds in an approximate sense (as in \cref{eq:lmat-basis-change}), then the same linear combination also holds approximately between the vectors $L_i \in \BR^{m}$. More precisely, we would want that for any coefficients $c_1, \ldots, c_n, c_1', \ldots, c_n' \in \BR$, we have
\begin{align}
 \sum_{i=1}^n c_i \cdot \Lmat(h_i, \cdot) \approx \sum_{i=1}^n c_i' \cdot \Lmat(h_i, \cdot) \quad \Leftrightarrow \quad \sum_{i=1}^n c_i \cdot L_i \approx \sum_{i=1}^n c_i' \cdot L_i.\label{eq:ideal-approx}
\end{align}
Unfortunately, it is straightforward to see that \cref{eq:ideal-approx} cannot hold for the above strategy for constructing the vectors $L_i$, even if we are lenient as to what sense the approximation ``$\approx$'' above holds. The issue arises because only observing a bounded number of samples from the conditional distributions given $h_i$ (i.e., $\model(\cdot \mid h_i)$) is not sufficient to learn some entries of $\Lmat(h_i, \cdot)$ corresponding to futures with extremely low probability.

This is an issue  more generally with any approach that aims to use samples from distributions to learn properties of new distributions obtained by performing ``linear algebra on their logits'' (as in \cref{eq:ideal-approx}). As a trivial example, consider the vectors $L = (M,0), L' = (M, M/2)$, for some $M \gg \log m$;  $m$ samples drawn from $\softmax(L)$ or $\softmax(L')$ will all be equal to $1$ with overwhelming probability, and so are insufficient to allow us to distinguish between, e.g., $\softmax(\frac{1}{M} \cdot L)$ and $\softmax(\frac 1M \cdot L')$, which are $\Omega(1)$-far in total variation distance. %

\paragraph{Our approach: using the elliptic potential lemma to \emph{adaptively} choose futures.}
The issue identified above seems quite serious: having only chosen the histories $h_1, \ldots, h_n$, it is not clear a priori \emph{which} values of $c_1, \ldots, c_n$ we need that \cref{eq:ideal-approx} holds for, and in order for a ``sampling-based'' approach such as the one discussed above to work, it seems that we would have to sample from the distribution induced by $\sum_{i=1}^n c_i \cdot \Lmat(h_i, \cdot)$ for any such tuple $(c_1, \ldots, c_n)$. Moreover, it is of course infeasible to iterate over all possibilities for $(c_1, \ldots, c_n)$.

Due to the infeasbility of deciding which tuples $(c_1, \ldots, c_n)$ we have to ``account for'' \emph{a priori}, we opt for the following alternative approach: rather than sampling the futures $f$ in advance, we will instead choose these futures \emph{adaptively}, based on the coefficients $c_1, \ldots, c_n$ that arise in the learning algorithm. In particular, having chosen some set of futures $\MF_t \subset \Sigma^{\leq T-t}$ at each step $t$, let us try to implement the sampling algorithm discussed in \cref{sec:exact-model-approx-oracle} where, whenever we encounter a vector $\Lmat(h_t, \cdot)$ (for some history $h_t \in \Sigma^{t-1}$ at step $t$), we replace it with the vector $(\Lmat(h_t, f))_{f \in \MF_t}$. Technically, we can only approximate this vector up to roughly error $\epapx$, since we have only assumed query access to an $\epapx$-approximate logit oracle $\Ologit$ for $\model$. This approximation will suffice for our needs. Our main insight is as follows:

\begin{idea}
  \label{idea:1}
We can run the sampling algorithm, as described above with our current approximations using $\calF_t$. Then we can verify whether the tokens we sampled were from the correct distribution; if they are, then we already have a good sampler, and if not, then we can add an element to  $\calF_t$ that adds a new ``dimension" (formalized with an elliptical potential argument described below).
\end{idea}
In more detail, the verification mentioned in \cref{idea:1} checks whether an (approximate) equality of the form \cref{eq:complete-induction} holds with respect to the \emph{entire} rows of the logit matrix $\Lmat(\Sigma^{t-1}, \Sigma^{T-t} \times \Sigma)$. %
Note that we can efficiently perform such a verification by sampling futures from $\model(\cdot \mid y_{1:t})$ and checking whether the corresponding logits agree approximately on the sampled futures. %
Moreover, if the logits do not agree (i.e., the verification fails), then one of the sampled futures can be added to $\MF_t$, and repeat the entire procedure. 

As described in \cref{idea:1}, if the verifications all pass (over multiple independent draws of $y_{1:T}$), then a simple concentration inequality yields that \cref{eq:complete-induction} indeed holds with high probability, which will imply we have a good sampler. %
But how do we ensure that this procedure will ever terminate, i.e., that we won't keep on adding additional futures to the set $\MF_t$ for more than a polynomial number of iterations? The key fact is to use the fact that $\model$ has logit rank $d$, which means that for each history $h_t \in \Sigma^{t-1}$ and and future $f_t \in \Sigma^{\leq T-t}$, there are vectors $\phi(h_t), \psi(f_t) \in \BR^d$, respectively, so that $\Lmat(h_t, f_t) = \langle \phi(h_t), \psi(f_t)\rangle$. If we have chosen coefficients $\hat c_{t+1}, \hat c_t$ so that \cref{eq:complete-induction} holds ``with respect to the current set of futures $\MF_t$'', then we have that for each $f_t \in \MF_t$,
\begin{align}
\left\langle \sum_{j=1}^{\ds} \hat c_{t+1,j} \cdot \phi(h_{t+1,j}), \psi(f_t) \right\rangle \approx \left\langle \sum_{i=1}^\ds \hat c_{t,i} \cdot \phi(h_{t,i} \circ y_t), \psi(f_t) \right\rangle\label{eq:holds-for-ft}
\end{align}
If on the other hand our ``verification'' process fails for \cref{eq:complete-induction}, it means that we can produce some future $ f_t^\st$ so that \cref{eq:holds-for-ft} fails to hold when $f_t$ is replaced by $ f_t^\st$, and then we can add $f_t^\st$ to $\MF_t$. 

Finally, we observe that it is a well-known consequence of the \emph{elliptic potential lemma} that this phenomenon can only happen (i.e., $f_t^\st$ as above exists and is added to $\MF_t$) roughly $\tilde O(d)$ times at each step $t$ (see \cref{lem:approximate-elliptic-potential} for a formal statement). Roughly speaking, the intuition for this fact is that one can only be ``surprised'' by ``new directions'' in $\BR^d$ roughly $\tilde O(d)$ times. We remark that very similar statements have been established in the context of, e.g., linear bandit problems (see e.g. \cite[Lemma 19.4]{lattimore2020bandit}) where this fact manifests as a bound on the \emph{eluder dimension} of linear functions.  
Similar phenomena of bounding the ``number of surprises'' to establish learning occur in reinforcement learning \cite{xie2022role} and sequential prediction \cite{block2024performance}.

\subsection{Controlling the growth of coefficients via linear programming}
\label{sec:lp-coefficients}
The previous subsection addresses the problem of choosing a ``representative'' set of futures that allows us to implement the procedure in \cref{sec:exact-model-approx-oracle} without having to operate on vectors with exponentially many entries. Now we return to the second issue mentioned at the end of \cref{sec:exact-model-approx-oracle}, namely that of the magnitude of the coefficient vectors $\hat c_t$ growing exponentially with $t$. %
  A similar issue of ``exponential blowup'' arises in the low-rank case \cite{liu2025model,mahajan2023learning}, but as we shall see, prior techniques fail in our low \emph{logit}-rank setting. 

In particular, \cite{liu2025model} handled this issue by means of a certain \emph{projection step} at each iteration $t$. 
In our setting, this technique would choose coefficients $\hat c_{t+1}$ at each step $t$ which are bounded in magnitude by, roughly speaking, minimizing the \emph{KL divergence} between the distributions over futures induced by the vectors $\sum_{j=1}^{\ds} \hat c_{t+1,j} \cdot \Lmat(h_{t+1,j}, \cdot)$ and $\sum_{i=1}^\ds \hat c_{t,i} \cdot \Lmat(h_{t,i} \circ y_t, \cdot)$.  %
When working with probability distributions directly (as in the low-rank case), this projection with respect to KL divergence can be implemented efficiently due to the convexity of the KL divergence. 
However, in our low logit-rank setting, it is necessary to exponentiate the logits of these vectors (and divide by appropriate normalizing constants) to obtain such distributions, so it is no longer clear how to efficiently implement such a projection step. One could attempt to use other notions of distance between distributions as opposed to KL divergence, but similar issues involving nonconvexity arise. Moreover,  because we can only perform computations directly on the polynomial-length vectors induced by only considering futures in the set $\MF_t \subset \Sigma^{\leq T-t}$ at each step $t$ (as discussed in \cref{sec:ep-dimred}), there are additional technical difficulties in ensuring that any such notion of distributional distance is preserved when we restrict to futures in $\MF_t$.

\paragraph{Our approach: ``sampling by linear programming''.} 

To address these issues, we propose a new approach to sample from some distribution $\hat\model$ which is close to $\model$: 
\begin{idea}
\label{idea:2}
Rather than trying to project the coefficients $\hat c_{t+1}$ at each step in a ``greedy'' manner, we instead solve a linear program which involves computing \emph{all} of the coefficient vectors $\hat c_1, \ldots, \hat c_{t+1}$ \emph{at each step $t$}
\end{idea}
To explain this linear program, suppose that we have chosen a a set of ``basis histories'' $h_{t,1}, \ldots, h_{t,\ds} \in \Sigma^{t-1}$ at each step $t$, as well as a set of ``representative futures'' $\MF_t \subset \Sigma^{\leq T-t}$ (recall from \cref{sec:ep-dimred} that $\MF_t$ will be periodically updated over the course of the algorithm). For each $t \in [T]$ and $i \in [\ds]$, let us define $L_{t,i} \in \BR^{\MF_t}$ by $L_{t,i}(f) = \Lmat(h_{t,i}, f)$ for $f \in \MF_t$. Then the vectors $L_{t,i}$ should be interpreted as ``representative'' vectors for the rows $\Lmat(h_{t,i}, \cdot)$ of the logit matrix $\Lmat$.

Suppose that we have already sampled a sequence of tokens $y_{1:t}$, and that we are attempting to sample $y_{t+1}$ from some distribution which is close to $\BP(y_{t+1} = \cdot \mid y_{1:t})$. To do so, we compute vectors $\hat c_1, \ldots, \hat c_{t+1} \in \BR^{\ds}$ and $\hat L_s \in \BR^{\MF_s}$ for each $s \in [t+1]$ solving the following linear program:
\begin{subequations}
\begin{align}
\hat L_s(f)= \sum_{i=1}^\ds \hat c_{s,i} \cdot L_{s,i}(f) & \quad \forall s \in [t+1], \  f \in \MF_s \label{eq:lp-intro-1}\\
\hat L_{s+1}(f) = \sum_{i=1}^{\ds} \hat c_{s,i} \cdot L_{s,i}(y_{s}, f) & \quad \forall s \in [t], \  f \in \MF_{s+1} \label{eq:lp-intro-2}\\
\| \hat c_s \|_\infty \leq 2  & \quad \forall s \in [t+1] \label{eq:lp-intro-3}.
\end{align}
\end{subequations}
Roughly speaking, the above program expresses similar relationships between the coefficient vectors $\hat c_1, \ldots, \hat c_{t+1}$ as in \cref{sec:exact-model-approx-oracle}, but we no longer compute the coefficients $c_{t,i,y_t,j}$ which allow us to move from step $s$ to step $s+1$ (see  \cref{eq:lmat-basis-change}) \emph{in advance}. Instead, we are directly computing the coefficient vectors $\hat c_s$ at all steps $s \in [t+1]$ simultaneously, which allows us to avoid the factor of $\ds$ blowup at each step incurred by the previous ``greedy'' approach.

One should interpret the variables $\hat L_s$ in \cref{eq:lp-intro-1,eq:lp-intro-2} as being an approximate ``representative'' vector for $\Lmat(y_{1:s}, \cdot)$ for each $s \in [t+1]$, in the sense that $\hat L_s(f) \approx \Lmat(y_{1:s}, f)$ for each $f \in \MF_s$. Since, as we will ensure, $\MF_s$ will contain all single-token futures, we will get that $\hat L_s(y) \approx \Lmat(y_{1:s}, y) = \logit(y_{s+1} \mid y_{1:s})$, for each $s \in [t+1]$. Specializing to $s=t+1$, we see that the softmax of $\hat L_{t+1}$ restricted to $\Sigma$ yields  a distribution which is close to $\softmax(\logit(y_{t+1} = \cdot \mid y_{1:t})) = \model(y_{t+1} = \cdot \mid y_{1:t})$.

Of course, the above argument only holds if the program is feasible. We will ensure that this is the case by an appropriate choice of the histories $h_{t,i}$ (see \cref{lem:feasibility}). The argument that a solution to the program yields a good approximation to the distribution of the next token is carried out in \cref{lem:most-covered,lem:sampling-accuracy}.

We have not yet addressed the issue of how the algorithm can actually learn the histories $h_{t,i} \in \Sigma^{t-1}$, for $t \in [T], i \in [\ds]$. This turns out to be quite straightforward: given some set of representative futures, we simply sample a large polynomial number of histories $h\^j \in \Sigma^{t-1}$ from $\model$, and compute a \emph{barycentric spanner} of the vectors $(\Lmat(h\^j, f))_{f \in \MF_t}$ for each $t \in [T]$; see \DistSpanner, \cref{alg:dist-spanner-simple}. It follows from standard concentration arguments that the resulting spanner will be good enough to ensure feasibility of the program \cref{eq:lp-intro-1,eq:lp-intro-2,eq:lp-intro-3} with high probability.

\subsection{Putting it all together: summary of the learning algorithm}
\label{sec:alg-overview-intro}
Finally, we summarize the overall structure of our learning algorithm. The below procedure simplifies several details; see \cref{alg:learning-approx} for full details. 
\begin{itemize}
\item The sets of futures $\MF_t$, $t \in [T]$, described in \cref{sec:ep-dimred} are initialized arbitrarily. Each time some future is added to some $\MF_t$, we restart the following procedure:
\begin{enumerate}
\item For each $t \in [T]$, we use the \DistSpanner algorithm (\cref{alg:dist-spanner-simple}) to construct a ``basis'' of histories $h_{t,1}, \ldots, h_{t,\ds}$ (for an appropriate choice of $\ds$) at step $t$. This allows us to define vectors $L_{t,i} \in \BR^{\MF_t}$, by $L_{t,i}(f) \approx \Lmat(h_{t,i}, f)$ ($f \in \MF_t$), as described in \cref{sec:lp-coefficients}.  (The approximation arises here since we can only estimate $\Lmat(h_{t,i}, f)$ up to $\epapx$ accuracy using the approximate logit oracle $\Ologit$ for $\model$.)
\item Using the vectors $(L_{t,i})_{t \in [T], i \in [\ds]}$, we repeat the following procedure a sufficiently large number of times:
\begin{enumerate}
  \item We sample a sequence of tokens $y_{1:T} \sim \model$.
\item We iterate over $t \in [T]$: at each time step $t$, we solve the linear program in \cref{eq:lp-intro-1,eq:lp-intro-2,eq:lp-intro-3} to compute coefficients $\hat c_1, \ldots, \hat c_{t+1}$ and vectors $\hat L_1, \ldots, \hat L_{t+1}$.
\item For each $t \in [T]$, we test if the vectors $\hat L_s$, $s \in [t+1]$, give good approximations of $\Lmat(y_{1:s-1}, \cdot)$ in roughly the sense of \cref{eq:ideal-approx} with respect to the coefficients $\hat c_s$ computed in the previous step, by sampling many futures from $\model(\cdot \mid y_{1:s-1})$. If this is not the case, we can find some $t \in [T]$ together with a future $f_t^\st$ which violates this approximation and thus can be added to $\MF_t$ as described in \cref{sec:ep-dimred}. In such a case, we restart the entire procedure.
\end{enumerate} 
\item If all of the tests pass in the previous step, then the algorithm terminates. At this point, we can use the vectors $L_{t,i}$ and the futures $\MF_t$ of the final iteration to implement the sampling procedure as described in \cref{sec:lp-coefficients} (which is essentially the same as the steps in the inner loop above, but where the token $y_{t+1}$ is generated using $\hat L_{t+1}$ at each step $t$ as opposed to being sampled from $\model$.)
\end{enumerate}
\end{itemize}

\subsection{Learning a language model with approximately low logit rank}
\label{sec:approx-model}
In the previous subsections we have described how to efficiently learn a language model $\model$ which is of low logit rank using access to an $\epapx$-approximate logit oracle $\Ologit$ for $\model$. It turns out that the same argument allows us to show that our algorithm works assuming only that the responses of the oracle $\Ologit$ are close to those of some low-rank logit matrices, instead of directly assuming that $\model$ is of low-logit rank (see \cref{def:bounded-logit-oracle}). To extend our result to the setting where $\model$ is only \emph{approximately} of low logit rank in the sense of \cref{def:avg-closeness-weak}, as discussed in the present section, we shall need this modified result. 

In more standard PAC learning settings (in which an algorithm is only allowed access to i.i.d.~samples from an unknown distribution $\mathbb{Q}$), obtaining results for learning with $\vep$-misspecification for polynomially small $\vep$ (e.g., in total variation distance) is typically straightforward. In particular, suppose an algorithm uses $n$ samples from an unknown distribution $\mathbb{Q}$ which is known to belong to some class $\MD$ of distributions. If instead the true distribution $\mathbb{Q}$ is only $\ll 1/n$-close to some distribution $\mathbb{Q}' \in \MD$, then the same algorithm will still succeed with high probability, since $n$ samples from each of $\mathbb{Q}$ and $\mathbb{Q}'$ can be coupled so that they are all equal with high probability. 

In the setting of learning with queries, such an argument, which guarantees learning with a modest amount of misspecification, no longer goes through. For instance, in our setting of learning with a logit oracle $\Ologit$, the learning algorithm could, in principle, query $\Ologit(y_{1:t})$ for some sequences $y_{1:t}$ with vanishingly small probability under $\model$. In such a case, if $\model$ is only promised to have logits close to those of some low-rank matrix for \emph{typical sequences} $y_{1:t} \sim \model$, then the logits for $\model$ on sequences $y_{1:t}$ queried by the algorithm \emph{might be very far from those of the low-rank matrix}! 

Indeed, in previous work on learning a low-rank distribution \cite{liu2025model}, such a pathological situation cannot be ruled out. In particular, the algorithm of \cite{liu2025model} makes queries to a conditional sampling oracle (which can be viewed as an analogue of our logit oracle $\Ologit$)\footnote{Formally, a \emph{conditional sampling oracle} $\Osamp$ for $\model$ takes as input $y_{1:t} \in \Sigma^t$ and returns a sample $y_{t+1} \sim \model(\cdot \mid y_{1:t})$.
} on certain sequences $y_{1:T}$ which are constructed inductively over the course of $T$ steps. Each step might choose some token $y_{t}$ (given $y_{1:t-1}$) with probability, say $\gg 2 \cdot \model(y_t \mid y_{1:t-1})$. The upshot is that the algorithm might query certain sequences $y_{1:T}$ with probability $\geq \exp(\Omega(T)) \cdot \model(y_{1:T})$. 

As we show in \cref{sec:avg-apx-proof}, our algorithm will not suffer from such an issue. In particular, we can couple the execution of \cref{alg:learning-approx} to the draw of \emph{polynomially many} i.i.d.~samples $y_{1:T}\^1, \ldots, y_{1:T}\^N$ from $\model$ so that each query to $\Ologit$ made by the algorithm is on a sequence of the form $y_{1:s-1}\^j \circ y_s' \circ y_{s+1:t-1}\^i \circ y_t'$, for some $s < t$ and $y_s', y_t' \in \Sigma$ (see \cref{lem:coupling-indep-draws}). In words, such a sequence is formed by concatenating subsequences two of the i.i.d.~sequences and appending an arbitrary token to each subsequence. Thus, our learning guarantee extends to any $\model$ whose logits are $\epavg$-close to those of a low-rank model on \emph{such ``concatenated'' sequences} (\cref{thm:approx-main}), which is exactly what \cref{def:avg-closeness-weak} requires. 

One might object that the resulting misspecification allowed is still somewhat small: we require that the amount of misspecification $\epavg$ satisfies $\epavg \leq 1/\poly(T, |\Sigma|, d)$, for a sufficiently large $\poly$ (specified in \cref{thm:approx-main}). We believe that our work nevertheless serves as a starting point for follow-up work on this problem. 
While it is straightforward to show that such inverse polynomial dependence on $T$ and $|\Sigma|$ is necessary, we are less sure about what whether inverse polynomial dependence on $d$ is required;  see \cref{sec:conclusions} for further discussion.

\section{Related Work} \label{sec:relatedwork}

Our work can be seen as part of a long line of research on learning sequence and latent variable models. 
Perhaps the most well-studied sequence models are hidden Markov models (HMMs). 
Similar to low logit rank models, noisy parities pose a barrier to learning HMMs from samples, which has motivated considering learning HMMs under natural but restrictive assumptions \cite{mossel2005learning, hsu2012spectral, balle2014spectral}. 
A major change in perspective came with the conditional query model considered by \cite{mahajan2023learning} which allowed the noisy parity barrier to be circumvented \cite{liu2025model}. 
Due to the close connection to our work, we discuss this line of work in later in the section. 
More broadly, there is a rich literature on learning latent variable models, such as mixtures of Gaussians (see \cite{liu2022clustering} and references therein), and POMDPs (see\cite{golowich2022planning,golowich2022learning,chen2024near} and references therein).

Of particular interest is the study of linear dynamical systems where state transitions are linear and observations are linear functions. 
These play a foundational role in control theory and time series analysis \cite{kailath1980linear, chen1999linear,hazan2025introductiononlinecontrol} and thus have been extensively studied across many communities. 
Interesting connection between linear dynamical systems and language models has been through the study of state space models
\cite{gu2021efficiently, gu2023mamba, katharopoulos2020transformers, dao2024transformers, gupta2022diagonal, gu2021combining} which has led several empirical successes in language modeling. 
State space models are closely related to low logit rank as discussed in \cite{golowich2025sequenceslogitsreveallow}.

Learning problems inspired from modern empirically successful models such as neural networks and transformers have been studied extensively in recent years (see \cite{allen2019learning,chen2025provably,misiakiewicz2023lectureslinearizedneuralnetworks, bartlett2021deeplearningstatisticalviewpoint} and references therein).
Similarly, there has been a long line of work on understanding the expressive power of models \cite{merrill2022saturated} and mechanistic interpretability \cite{bereska2024mechanistic}, with the aim of understanding the inner workings of modern language models.
Unfortunately, due to the inherent complexity of these models, in order to facilitate theoretical analysis, these works often consider extremely simplified settings such as bounded depth, restricted architechtures and scaling limits, and thus fail to capture modern language models in a realistic manner.

Another interesting line of work that has been proposed to understand language models from a theoretical perspective is language generation in the limit \cite{kleinberg2024language}. 
The insight here is that the task of language generation (seen in general formal language setting) can be achieve even though ``finding'' the underlying model is intractable. 
Starting from this, the framework has been used to give \emph{qualitative} insights into language modeling \cite{kalavasis2025limits,karbasi2025possibility}, but due to the abstract nature of the framework, cannot directly be used to reason about modern language models.

Finally, we note that learning from queries has a long tradition in learning theory \cite{angluin1987learning, kushilevitz1991learning}. 
More recently, conditional queries have been studied in property testing \cite{chakraborty2013power} and learning theory \cite{mahajan2023learning,liu2025model} where they have been shown to circumvent statistical and computational barriers.

\paragraph{Low Rank Models} 
Perhaps the most closely related line of work to our is the recent study of low rank language models \cite{liu2025model,mahajan2023learning}.
Motivated by the study of algorithms for learning HMMs, the works consider the rank of the matrix of conditional probabilities of futures given histories.
Formally, a low rank language model is defined as follows.

\begin{definition}[Low Rank Language Model]
A low rank language model $\model$ has rank $d$ if for any subsets $\calH, \calF \subseteq \Sigma^*$, the matrix given by
\[
\{ \model[f|h] \}_{h \in \calH, f \in \calF}
\]
has rank at most $d$.
\end{definition}

As noted above, low rank language models captures HMMs as a special case since a HMM with $d$ hidden states has rank at most $d$.
Similar to our earlier definitions, this notion involves an exponentially large matrix with rows being indexed by histories and columns indexed by futures being low rank.  
In fact, we can easily verify that if for $f = y_{t+1:t+s}, h = y_{1:t}$, we define $\logit_{\model}[f|h] =  \logit_{\model}[y_{t+s}|y_{1:t+s-1}] + \dots + \logit_{\model}[y_{t+1}|y_{1:t}]$
then a low logit rank model has the property that the matrix with entries $\{\logit_{\model}[f|h]\}_{h \in \calH, f \in \calF}$ is low rank.

While this notion was an inspiration for our work, it has crucial limitations for approximating modern language models. 
The key difference is that this low rank is in logit space rather than probability space.  
This seemingly minor technical difference is crucial for being able to approximate real language models.  
It is easily verified experimentally that modern language models are \emph{not close to low rank in probability space}.  
In particular, for any reasonable choice of $n$, if we sample $n$ histories $h_1, \dots , h_n$, then the distributions on futures given by $\model[\cdot | h_1], \dots , \model[\cdot | h_n]$ are essentially disjoint. 
Therefore, the different rows (in the probability matrix) corresponding to $h_1, \dots , h_n$ essentially have disjoint support and thus there cannot be any linear dependencies.
This example illustrates a key difference between low rank and low logit rank models, and as formalized in \cite{golowich2025sequenceslogitsreveallow}, for tasks such the above ``recall" task, low logit rank can give exponentially more efficient representations compared to low rank models, which is a crucial property that language models seem to possess empirically and thus might serve as starting point for understanding the success of modern language models.

\section{Algorithm} \label{sec:algo} 

In this section, we describe the algorithm which establishes \cref{thm:approx-main-informal}. First, in \cref{sec:prelims}, we will discuss some preliminaries pertaining to the logit oracles we use; we then give the formal algorithm description in \cref{sec:algorithm-formal}. 

\subsection{Preliminaries on the oracles} \label{sec:prelims} 

Recall that we consider \emph{autoregressive (logit) oracles}, namely functions which take as input a sequence of tokens and output (estimates) of the logits for next token. %

\cref{def:exact-logit-oracle} introduced an \emph{exact logit oracle}, whose answers gave the next token logits exactly; our results additionally hold under the weaker notion of \emph{approximate logit oracle} in \cref{def:approx-logit-oracle} below. 
We let $\projOne$ denote the ``mean-centering'' projection: $\projOne(v) = v - \frac{1}{|\Sigma|} (\mathbf{1}^\t v) \cdot \mathbf{1}$ for any $v \in \BR^{|\Sigma|}$. %
\begin{definition}[$\ep$-approximate logit oracle]
  \label{def:approx-logit-oracle}
For $\ep > 0$, an \emph{$\ep$-approximate logit oracle} $\Ologit$ takes as input $y_{1:t} \in \Sigma^t$ a returns a vector $\Ologit(y_{1:t}) \in \BR^\Sigma$ satisfying $\| \projOne\Ologit(y_{1:t}) -\logit_\model(y_{t+1} = \cdot \mid y_{1:t}) \|_\infty \leq \ep$. 
\end{definition}
Further, our algorithm will use a \emph{sampling oracle} for $\model$, which simply outputs samples $y_{1:T} \sim \model$. Note that such an oracle can be implemented (up to $\ep T$ total variation distance error) using $T$ queries to an $\ep$-approximate logit oracle (\cref{def:approx-logit-oracle}); for simplicity, we allow direct access to such a sampling oracle.

In addition to requiring that a model's logits be low rank, we need some mild boundedness conditions on the low-rank decompositions of these matrices. To state them, we first make the following technical definition regarding the boundedness of the factors for a low-rank factorization of a matrix: %
\begin{definition}
\label{def:bounded-logit}
Fix $\alpha > 0$ and $d \in \BN$ For $t \in [T]$, A matrix $\logitmatrix \in \BR^{\Sigma^t \times (\Sigma^{\leq T-t-1} \times \Sigma)}$ is defined to be \emph{$\alpha$-bounded of rank $d$} if for each $h \in \Sigma^t, f \in \Sigma^{\leq T-t-1}, y \in \Sigma$, there are vectors $\Bhist_h, \Bfut_{f,y} \in \BR^d$ satisfying $\| \Bhist_h \|_2 \leq \alpha, \| \Bfut_{f,y} \|_2 \leq \alpha$ so that
$
\logitmatrix_{h,(f,y)} = \langle \Bfut_{f,y}, \Bhist_h \rangle\nonumber.
$
To keep notation uncluttered, when considering such $\logitmatrix$, we often denote the $(h, (f,y))$ entry of $\logitmatrix$ as $\logitmatrix(h, (f,y))$. 
\end{definition}

Next, we strengthen the condition of \cref{def:avg-closeness-weak} to require that the matrices $\logitmatrix\^s$ there are actually $\alpha$-bounded of rank $d$.
\begin{definition}
\label{def:avg-closeness-true}
Fix $\ep, \alpha > 0$ and a distribution $\model \in \Delta(\Sigma^T)$. We say that $\model$ has \emph{$\ep$-approximate $\alpha$-bounded logit rank $d$} if there are matrices $\logitmatrix\^s$ for $0 \leq s \leq T$ satisfying the conditions of \cref{def:avg-closeness-weak} which are moreover $\alpha$-bounded of rank $d$ (per \cref{def:bounded-logit}).
\end{definition}
We remark that our query and computational cost bounds will depend only \emph{logarithmically} on the bound $\alpha$, meaning that requiring $\model$ to satisfy \cref{def:avg-closeness-true} is a very mild condition over that of \cref{def:avg-closeness-weak}. For instance, if all of the entries of the rank-$d$ matrix $\logitmatrix\^s$ are bounded by $\lambda$, then we can take $\alpha = O(\lambda \cdot |\Sigma|^T)$, so that $\log \alpha = O(T \log |\Sigma| + \log \lambda)$.

Finally, as a tool to prove our main result, we will use the below definition, which describes when a logit oracle returns logit estimates which are low-rank and bounded in the sense of \cref{def:bounded-logit}:
\begin{definition}[Approximately bounded low-rank logit oracle]
  \label{def:bounded-logit-oracle}
Fix $\ep, \alpha > 0$ and $d \in \BN$ as well as a logit oracle $\Ologit$. Suppose an algorithm calls $\Ologit$ on a collection of sequences $y\^1_{1:t_1}, \ldots, y\^N_{1:t_N} \in \Sigma^\st$ (perhaps chosen adaptively), for some $n \in \BN$ and $t_1, \ldots, t_n \in [T]$. We say that the \emph{execution trace of $\Ologit$ is $\ep$-approximately $\alpha$-bounded of rank $d$} if for each $s \in [T]$ there is a matrix $\logitmatrix\^s \in \BR^{\Sigma^s \times (\Sigma^{\leq T-s-1} \times \Sigma)}$ which is $\alpha$-bounded of rank $d$ (\cref{def:bounded-logit}) so that for each $i \in [N]$ with $t_i > s$, we have $\| \projOne(\Ologit(y\^i_{1:t_i-1})) - \logitmatrix\^s(y\^i_{1:s}, (y\^i_{s+1:t_i-1},\cdot))\|_\infty \leq \ep$.
\end{definition}

\paragraph{Miscellaneous notation.} %

Let $[N] := \{1,\dots,N\}$.
We write $\|v\|_2$ for the Euclidean norm and $ \| v \|_{\infty} $ for the max norm.
For a set $S$, we use $\Delta(S)$ to denote the probability simplex over $S$. 
For a distribution $Q \in \Delta([N])$ and $v \in \BR^N$, we define the $Q$-weighted norm $\|v\|_{2,Q} := \sqrt{\sum_{i=1}^N Q(i) v(i)^2}$. 

\subsection{Algorithm description}
\label{sec:algorithm-formal}
In this section, we describe our main algorithm for proving \cref{thm:infty-oracle-main}, which establishes efficient learnability of a distribution $\model$ using access to samples from $\model$ together with an $\epapx$-approximate logit oracle $\Ologit$ for $\model$ (\cref{def:approx-logit-oracle}).\footnote{Note that $\Ologit$ allows us to produce samples from a distribution which is $\epapx \cdot T$-close in total variation distance to $\model$; accordingly, we could in fact assume only access to $\Ologit$, though for convenience we also assume the ability to directly sample from $\model$.} Throughout this section and the following ones, we denote the logits of $\model$ by $\logit(y_{1:t}) := \logit_\model(y_{1:t})$. Rather than directly assuming that the distribution $\model$ has low logit rank (\cref{def:logit-rank-exact}), we instead prove the following in \cref{thm:infty-oracle-main}: \emph{under the event $\Eoracle$ that all of the algorithm's queries to $\Ologit$ are approximately consistent with a low-rank matrix}, then the algorithm successfully learns $\model$ (except on a failure event of small probability). We remark that we could ensure that $\Pr(\Eoracle) = 1$ by, e.g., assuming that $\model$ has low logit rank and that $\| \projOne \Ologit(y_{1:t}) - \logit(y_{t+1} = \cdot \mid y_{1:t}) \|_\infty \leq \epapx$, for some small $\epapx > 0$. However, the greater generality afforded by the phrasing of \cref{thm:infty-oracle-main} with the event $\Eoracle$ allows us to extend its guarantee, in \cref{sec:avg-apx-proof}, to the case when the distribution $\model$ is only approximately low logit rank in the sense of \cref{def:avg-closeness-true} (\cref{thm:approx-main}). 

Our algorithm is split into two parts: first, \cref{alg:learning-approx} (the ``learning algorithm'') uses the ability to draw samples $y_{1:T} \sim \model$ for an unknown language model $\model$, as well as an $\epapx$-approximate oracle $\Ologit$ for $\model$ (\cref{def:approx-logit-oracle}) to produce a succinct representation of a distribution $\hat \model$ which approximates $\model$ (under the event $\Eoracle$). Then, \cref{alg:sampling-approx} (the ``sampling algorithm'') uses this representation to efficiently produce samples from $\hat \model$ (i.e., without any access to $\model$ or $\Ologit$); in particular, the definition of $\hat\model$ given this succinct representation is given by \cref{alg:sampling-approx}.

\begin{algorithm}[H]
  \caption{Learning algorithm for low logit rank language models}
  \label{alg:learning-approx}
  \begin{algorithmic}[1]\onehalfspacing
    \Require $\epapx$-approximate logit oracle $\Ologit$ for a distribution $\model$, which induces $\Lapx$ (see \cref{eq:lapx-def}), as well as a sampling oracle for $\model$. Parameters $\beta, \gamma, \gamthres, K, n, \eta, \delta$ (see \cref{sec:parameter-settings}). 
    \State \label{line:initialize-ft} For each $t \in [T]$, initialize sets of futures $\hat \MF_t \subset \Sigma^{\leq T-t+1}$ by $\hat \MF_t \gets  \{ y_t\}$ for a sample $y_{1:T} \sim \model$. %
    \For{$1 \leq k \leq K$}
    \State \label{line:define-tildeqt}  For $t \in [T]$, define $\tilde \MF_t = \hat \MF_t \cup (\Sigma \circ \hat\MF_{t+1})$, and define $\tilde Q_t := \Unif(\tilde \MF_t)$, %
     $\ds := \max_{t \in[t]} |\tilde \MF_t|$. 
     \State \multiline{For $t \in [T]$, let $\MP_t \in\Delta(\BR^{\tilde \MF_t})$ denote the distribution over vectors $L \in \BR^{\tilde \MF_t}$ as follows: we draw $y_{1:t-1} \sim \model$, and then we define $L_f = \Lapx( y_{1:t-1} \circ f)$ for $f \in \tilde \MF_t$.\label{line:define-pt}}
    \State Let $L_{1,1} = \cdots = L_{1,\ds} = \Lapx(\cdot) \in \BR^{\tilde \MF_1}$.

    \For{$2 \leq t \leq T$}
    \State \multiline{Call \DistSpanner (\cref{alg:dist-spanner-simple}) with the parameters $\eta, \delta/(10KT)$ and $\MP_t$. Let the returned vectors (padded with duplicates so that there are $\ds$ of them) be denoted by $L_{t,1}, \ldots, L_{t,\ds} \in \BR^{\tilde \MF_t}$, with $L_{t,i}(f) = \Lapx(h_{t,i} \circ f)$ for some $h_{t,i} \in \Sigma^{t-1}$. \label{line:dist-spanner-algo}} %
    
    \EndFor
    \State Sample $y_{1:T,u} \sim \model$ for each $u \in [n]$. \label{line:sample-ytu}
    \For{$1 \leq t \leq T-1$}\label{line:inner-for-learning}
    \State \multiline{For each $u \in [n]$ and $s \in [t]$, compute vectors $c_{t,s,u} \in \BR^\ds$ and for $s \in [t+1]$, vectors $\hat L_{t, s,u}\in \BR^{\hat \MF_s}$ solving \label{line:solve-program-algo} %
    \begin{subequations}
    \begin{align}
  \hat L_{t,s,u}(f) = \sum_{i=1}^\ds c_{t,s,u,i} \cdot L_{s,i}(f) \quad \forall f \in \hat \MF_s,\ u \in [n],\  1 \leq s \leq t \label{eq:hatl-l-close}\\
         c_{t,1,u} = e_1 \in \BR^\ds, \qquad   \| c_{s,u} \|_\infty \leq \cnorm\quad \forall 1 \leq s \leq t\label{eq:cnorm-bound-alg}\\
      \hat L_{t, s+1, u}(f) = \sum_{i=1}^\ds c_{t,s,u,i} \cdot L_{s,i}(y_{s,u} \circ f) \quad \forall f \in \hat \MF_{s+1},\ u \in [n], 1 \leq s \leq t\label{eq:join-steps}.
    \end{align}
  \end{subequations}}
    \If{The program \cref{eq:hatl-l-close,eq:cnorm-bound-alg,eq:join-steps} is infeasible for some $u \in [n]$}
    \State Continue to the next value of $k$.
    \EndIf
  \State \multiline{For $f \in \Sigma^{\leq T-s+1}$, we extend the above notation as follows: $L_{s,i}(f): = \Lapx(h_{s,i} \circ f)$ and $\hat L_{t,s,u}(f) := \sum_{i=1}^\ds c_{t,s-1,u,i} \cdot L_{s-1,i}(y_{s-1,u} \circ f)$.\label{line:extend-vecs}}
    \If{There are $s,r \in [t+1],\ u \in [n]$ with $r \geq s$ and $y_r' \in \Sigma$ s.t
$
        \left| \hat L_{t,s,u}(y_{s:r-1,u} \circ y_r') - \sum_{i=1}^\ds c_{t,s,u,i} \cdot L_{s,i}(y_{s:r-1,u} \circ y_r')\right| > \gamthres\nonumber
$
    }\label{line:test-discrepancy}
    \State Set $\hat \MF_s \gets \hat \MF_s \cup \{y_{s:r-1,u} \circ y_r'\}$ for some such $s, r, y_r'$.\label{line:add-to-qs}
    \State \textbf{Continue} to the next epoch $k$.
    \EndIf
\State \multiline{Using $\Ologit$, compute and store the values of $\Lapx(h_{t,i} \circ y)$ for each $t \in [T], i \in [\ds], y \in \Sigma$; by abuse of notation, we denote these values by $L_{t,i}(y)$.\label{line:final-logits}}
\State \multiline{\Return the spanners $L_{t,1:\ds}$, the sets $\hat \MF_t, \tilde \MF_t$, $t \in [T]$, and the values $L_{t,i}(y)$ from the previous step.\label{line:return-learning}}
    \EndFor

    \EndFor
  \end{algorithmic}
\end{algorithm}

\begin{algorithm}[H]
  \caption{Sampling algorithm for low logit rank language models}
  \label{alg:sampling-approx}
  \begin{algorithmic}[1]\onehalfspacing
    \Require For $1 \leq t \leq T$, sets $\hat \MF_t, \tilde \MF_t$ satisfying $\Sigma \subset \hat \MF_t \subset \tilde \MF_t \subset \Sigma^{\leq T-t+1}$, vectors $L_{t,1}, \ldots, L_{t,\ds} \in \BR^{\tilde \MF_t}$, for some $\ds \in \BN$, as well as values $L_{t,i}(y)$ for $t \in [T], i \in [\ds], y \in \Sigma$.
    \State Define $\hat P_1 \in \Delta(\Sigma)$ by $\hat P_1(y) \propto \exp(L_{1,1}(y))$; sample $y_1 \sim \hat P_1$. 
    \For{$1 \leq t \leq T-1$}
    \State \multiline{Compute vectors $c_{t,s} \in \BR^\ds$ and $\hat L_{t,s} \in \BR^{\hat \MF_s}$ solving the convex program in \cref{eq:hatl-spanner,eq:cnorm-bound,eq:hatl-next} defined by $\hat \MF_{1:t}, L_{1:t,1:\ds}, y_{1:t}$.}
    \If{the program is infeasible}
    \State \Return \textbf{fail}.
    \EndIf
    \State \multiline{Define $\hat L_{t,t+1}(y) = \sum_{i=1}^\ds c_{t,t,i} \cdot L_{t,i}(y_t \circ y)$ for all $y \in \Sigma$, using the values of $L_{t,i}(y)$ passed as input to the algorithm.\label{line:extend-lhat}}
    \State Define $\hat P_{t+1} \in \Delta(\Sigma)$ by $\hat P_t(y) \propto \exp(\hat L_{t,t+1}(y))$.
    \State Sample $y_{t+1} \sim \hat P_{t+1}$. 
    \EndFor
  \end{algorithmic}
\end{algorithm}

\paragraph{Description of the learning algorithm.} \cref{alg:learning-approx} takes as input an $\epapx$-approximate logit oracle $\Ologit : \Sigma^{\leq T} \to \BR$ (\cref{def:approx-logit-oracle}). %
For any sequence $y_{1:t} \in \Sigma^t$, by simply calling $\Ologit(y_{1:t-1}) \in \BR^\Sigma$ and  mean-centering the results, the algorithm can compute an estimate of the mean-centered logits $\logit(y_{1:t}) = \logit(y_t \mid y_{1:t-1})$. We formalize this by defining
\begin{align}
\Lapx(y_{1:t}) := \projOne( \Ologit(y_{1:t-1}))_{y_t} %
\label{eq:lapx-def},
\end{align}
so that the assumption that $\Ologit$ is an $\epapx$-approximate logit oracle implies that $$\E_{y_{1:t-1} \sim \model} \left[ \max_{y_t \in \Sigma}|\Lapx(y_{1:t}) - \logit(y_{1:t})|\right] \leq \epapx$$ for all $y_{1:t} \in \Sigma^t.$

For a parameter $K \in\BN$, \cref{alg:learning-approx} operates in $K$ \emph{epochs}. In each epoch, a set of futures $\hat \MF_t \subset \Sigma^{\leq T-t+1}$ is fixed at each step $t \in [T]$. Setting $\tilde \MF_t = \hat\MF_t \cup (\Sigma \circ \hat \MF_{t+1})$ (where $\Sigma \circ \MF$ means $\{ y \circ f \mid f \in \MF \}$), we define a distribution $\MP_t$ over vectors in $\BR^{\tilde \MF_t}$ by drawing $y_{1:t-1} \sim \model$ and defining the vector $L \in \BR^{\tilde \MF_t}$ by $L_f = \Lapx(y_{1:t-1} \circ f)$ for each $f \in \tilde \MF_t$ (\cref{line:define-pt} of \cref{alg:learning-approx}). Roughly speaking, one can think of $\tilde \MF_t$ as denoting a ``representative set'' of futures so that constructing a distributional spanner for $\MP_t$ gives us a ``good'' basis of histories for the \emph{entire} logit matrix (or else helps us to identify an element that should be added to some set $\hat\MF_t$ for future epochs).

Using \DistSpanner (\cref{alg:dist-spanner-simple}), we next compute a set of vectors $L_{t,1}, \ldots, L_{t,\ds} \in \BR^{\tilde \MF_t}$ which form an $(\eta, \delta/(10KT))$-distributional spanner for $\MP_t$ with respect to the uniform distribution over $\tilde \MF_t$ (\cref{line:dist-spanner-algo}). Here $\ds$ denotes an upper bound on the sizes of the sets $\tilde\MF_s$. At each step $t$, we then solve a linear program (in \cref{eq:hatl-l-close,eq:cnorm-bound-alg,eq:join-steps}) multiple times, corresponding to independently drawn trajectories $y_{1:t,u} \sim \model$, for each $u \in [n]$ (with an appropriate choice of $n$). As explained in \cref{sec:lp-coefficients} this linear program emulates our sampling procedure, except that the tokens $y_{1:t,u}$ are drawn from $\model$ as opposed to being produced by the sampling procedure itself. %

For each $t \in [T]$, the linear program in \cref{eq:hatl-l-close,eq:cnorm-bound-alg,eq:join-steps} produces vectors $\hat L_{t,s,u} \in \BR^{\hat\MF_s}$, for $s \in [t+1]$, which should be interpreted as approximations to the vector $(\logit(y_{1:s-1,u} \circ f))_{f \in \hat\MF_s}$. In \cref{line:test-discrepancy}, we then test whether this approximation \emph{generalizes} beyond futures in $\hat\MF_s$, in the sense that it holds on futures $y_{s:r-1,u}$ drawn from the same trajectory indexed by $u$. If this approximation does not generalize for some $t \in [t], s \in [t+1]$, as witnessed by some future $y_{s:r-1,u} \circ y_r'$, then we add that future to $\hat\MF_s$ (\cref{line:add-to-qs}) and move on to the next epoch $k$. 

Otherwise, we can be confident that the current sets of futures $\hat\MF_t$ contain ``all the necessary information about the full logit matrix''. This allows us to perform the following steps: First, we store the values of $\Lapx(h_{t,i} \circ y)$ for all $t \in [T], i \in [\ds], y \in \Sigma$ (\cref{line:final-logits}), and given a solution to the program \cref{eq:hatl-l-close,eq:cnorm-bound-alg,eq:join-steps}, we can define $\hat L_{t,t+1,u}(y)$ for any $y \in \Sigma$ by setting $f = y$ in \cref{eq:join-steps}. Then with high prbability we will have that the distribution $\hat P_t(y) \propto \exp(\hat L_{t,t+1,u}(y))$ is a good approximation to the true conditional distribution $\model(y \mid y_{1:t,u})$. With this in mind, the algorithm returns all of the necessary information to reproduce this sampling procedure (\cref{line:return-learning}), namely the sets $\hat\MF_t \subset \tilde \MF_t \subset \Sigma^{T-t+1}$ for $t \in [T]$, the vectors $L_{t,i} \in \BR^{\tilde \MF_t}$ for $t \in [T], i \in [\ds]$, and additionally the values $L_{t,i}(y)$, for $t \in [t], i \in [ds], y \in \Sigma$ as computed in \cref{line:final-logits}. 

\paragraph{Description of the sampling algorithm.} The information returned by \cref{alg:learning-approx} is enough to sample from some distribution $\hat\model$ (which, as we will show in our proof, is close in total variation distance to $\model$). The full sampling procedure is shown in \cref{alg:sampling-approx}. It mimics the procedure in the \textbf{for} loop in \cref{line:inner-for-learning}, except that the tokens $y_{1:t}$ are now generated by the sampling procedure itself, rather than being drawn from $\model$. In particular, at each step $t$, having computed the vector $\hat L_{t,t+1} \in \BR^{\hat \MF_s}$, we extend it to have values at $\hat L_{t,t+1}(y)$ for $y \in \Sigma$ by using the values of $L_{t,i}(y)$ passed as input to the sampling algorithm \cref{line:extend-lhat}  and use the resulting values of $\hat L_{t,t+1}(y)$ to sample token $y_{t+1}$. Per the discussion above the resulting distribution will be close to $\BP(y_{t+1} = \cdot \mid y_{1:t})$.

\section{Analysis of the Algorithm} \label{sec:analysis} 

\subsection{Technical preliminaries}
In this section we introduce several technical results which are well-known in the literature.
\begin{definition}
  \label{def:bounded-rankd}
 Given $B,R, \epapx > 0$, we say that a collection of vectors $L_1, \ldots, L_m \in \BR^N$ is \emph{$(B,R,\epapx)$-bounded of rank $d$} with respect to a sequence of indices $I_1, \ldots, I_m \in [N]$ if there is a matrix $\Phi \in \BR^{N \times d}$ and vectors $v_1, \ldots, v_m \in \BR^d$ so that $|(L_j)_{I_k} - (\Phi \cdot v_j)_{I_k}| \leq \epapx$ for all $k \leq j \leq m$ and:
  \begin{enumerate}
  \item Each row of $\Phi$ has $\ell_2$ norm at most $R$.
  \item $\| v_j \|_2 \leq B$ for all $j \in [m]$.
  \end{enumerate}
\end{definition}

The following lemma is a consequence of the \emph{elliptic potential lemma}, and plays a key role in establishing that our main learning algorithm must eventually terminate. It states that for a sequence of vectors in some high-dimensional space $\BR^N$ (where $N$ should be interpreted as the number of futures, which in particular is exponentially large) which are approximately low-dimensional in the sense of \cref{def:bounded-rankd}, we can only be ``fooled'' so many times by new coordinates of these vectors being large in magnitude. For completeness, we provide a full proof, which is standard. 
\begin{lemma}
  \label{lem:approximate-elliptic-potential}
There is a sufficiently large constant $C_{\ref{lem:approximate-elliptic-potential}} > 0$ so that the following holds.  Fix integers $d, N \in \BN$ and $B,R,\gamma > 0$. Suppose $L_1, \ldots, L_T \in \BR^N$ and indices $J_1, \ldots, J_T \in [N]$ are given so that $L_1, \ldots, L_T$ are $(B,R,\gamma/2)$-bounded of rank $d$ with respect to the indices $J_1, \ldots, J_T$. Write $Q_t := \frac{1}{t-1} \sum_{s=1}^{t-1} \delta_{J_s} \in \Delta(N)$. Suppose that for each $t \in [T]$, $|L_t|_{J_t} \geq \gamma \cdot \sqrt{C_{\ref{lem:approximate-elliptic-potential}}d \log(RB/\gamma)}$ but $\| L_t \|_{2,Q_t} \leq \gamma$. Then $T \leq C_{\ref{lem:approximate-elliptic-potential}} d \log(BRd/\gamma)$.
\end{lemma}
\begin{proof}
  Let $\Phi\in \BR^{N\times d}$ and $v_t\in\BR^d$ with $\|v_t\|_2\le B$ be such that $|(L_t)_{J_s} - (\Phi v_t)_{J_s}|\leq \gamma/2$ for all $s \leq t\leq T$ and the rows $\phi_j^\t$ of $\Phi$ satisfy $\|\phi_j\|_2\le R$. Set $\lambda := \gamma^2/B^2$ and define $M_0:=\lambda I_d$, $M_t:=M_0 + \sum_{s=1}^{t} \phi_{J_s} \phi_{J_s}^\t$. The elliptical potential lemma \cite[Lemma 19.4]{lattimore2020bandit} gives
  \begin{align}
    \sum_{t=1}^T \min\{1,\,\phi_{J_t}^\t M_{t-1}^{-1}\phi_{J_t}\} \le 2d\,\log\Big(1 + \frac{TR^2}{\lambda d}\Big) \leq 4d \log \left(\frac{TR^2B^2}{d\gamma^2} \right) .\label{eq:ep-upper}
  \end{align}
  On the other hand, using $|L_t|_{J_t} \geq \gamma\cdot \sqrt{C d \log(RB/\gamma)}$ and $|(L_t)_{J_t} - (\Phi v_t)_{J_t}|\leq \gamma/2$,
  \[
    |\phi_{J_t}^\t v_t| \;\geq\; |L_t(J_t)| - |L_t(J_t) - \phi_{J_t}^\t v_t|
    \;\geq\; \gamma\cdot \sqrt{C d \log(RBd/\gamma)} - \gamma/2 \;\geq\; \gamma\cdot \sqrt{C d \log(RBd/\gamma)} / 2.
  \]
  Moreover, by Cauchy--Schwarz in the $M_{t-1}$-inner product and using \[
 \frac{1}{t-1}\sum_{s=1}^{t-1}(\phi_{J_s}^\t v_t)^2\leq    2 \|L_t\|_{2,Q_t}^2 + 2\| L_t - \Phi v_t \|_{2,Q_t}^2   \le 3 \gamma^2,
  \]
  we have
  \begin{align}
    (\phi_{J_t}^\t v_t)^2 \le (v_t^\t M_{t-1} v_t)\, (\phi_{J_t}^\t M_{t-1}^{-1}\phi_{J_t}) \le (\lambda B^2 + 3(t-1)\gamma^2)\, (\phi_{J_t}^\t M_{t-1}^{-1}\phi_{J_t}).\nonumber
  \end{align}
Thus, for each $t \in [T]$, we deduce
  \begin{align}
    \phi_{J_t}^\t M_{t-1}^{-1}\phi_{J_t} \ge \frac{\gamma^2 \cdot Cd \log(RBd/\gamma)/4}{3t\gamma^2} = \frac{Cd \log(RBd/\gamma)}{12t} .\label{eq:ep-lower-one}
  \end{align}
For sufficiently large $C$, summing \cref{eq:ep-lower-one} over $t \in [T]$ yields a contradiction to \cref{eq:ep-upper}.
\end{proof}

Next, we shall need the notion of \emph{barycentric spanner} of a set of vectors, which is a set of vectors (typically taken to be a subset of the set) so that any element of the set can be written as a linear combination of the spanner elements with bounded coefficients.
\begin{definition}
  \label{def:bar-spanner}
  For $N,d \in \BN$, $\beta$, and a set $\MS \subset \BR^N$, a collection of vectors $v_1, \ldots, v_d$ is defined to be a \emph{$\beta$-barycentric spanner} for $\MS$  if for any $w \in \MS$, there is a vector $c \in \BR^d$ with $\| c \|_\infty \leq \beta$ for which
  \begin{align}
w = \sum_{i=1}^d c_i \cdot v_i\nonumber.
  \end{align}
\end{definition}

The below notion of \emph{distributional spanner} is related to \cref{def:bar-spanner} but only requires that most of the mass of a distribution over vectors can be approximated by a bounded linear combination of the spanner elements.
\begin{definition}[Distributional spanner]
  \label{def:dist-spanner}
  Consider a distribution $\MP \in \Delta(\BR^N)$ over vectors  $L \in \BR^N$. For $Q \in \Delta(N)$, we say that a set $\MB = \{ B_1, \ldots, B_d \} \subset \BR^N$ is a $(\eta, \beta)$-\emph{distributional spanner} for $\MP$ if with probability at least $1-\eta$ over $L \sim \MP$, there exists $c \in \BR^d$ with $\| c \|_\infty \leq \beta$ so that 
  \begin{align}
L = \sum_{i=1}^d c_i \cdot B_i\nonumber.
  \end{align}
\end{definition}

\begin{algorithm}[H]
  \caption{$\mathsf{DistSpanner}$: construction of distributional spanner}
  \label{alg:dist-spanner-simple}
  \begin{algorithmic}[1]\onehalfspacing
    \Require Distribution $\MP \in \Delta(\BR^N)$, parameters $\eta,  \delta > 0$. 
    \State For $m = \Theta((N/\eta)^{2} \cdot \log(N/(\delta\eta)))$, draw $m$ samples $w_1, \ldots, w_m \sim P$.\label{line:draw-spanner-sample}
    \State Find a subset $I \subset [m]$ of size $N$ so that $\{ w_i \mid i \in I \}$ is a $(2,0)$-approximate barycentric spanner of $\{ w_i \mid i \in [m]\}$. \emph{(Using, e.g., the algorithm of \cite{awerbuch2008online}).}
    \State \Return $\{ w_i \mid i \in I \}$. 
  \end{algorithmic}
\end{algorithm}

The following lemma shows that sampling a sufficiently large number of vectors from a distribution and finding a barycentric spanner of the sampled vectors yields a distributional spanner for the original distribution. We remark that a barycentric spanner of a set of $m$ vectors in $\BR^N$ can be found in $\poly(m,N)$ time (e.g., using the algorithm of \cite{awerbuch2008online}).
\begin{lemma}[Distributional spanner; \cite{liu2025model}]
  \label{lem:dist-exact-RN}
Suppose that $\MP \in \Delta(\BR^N)$ is a distribution over vectors $L \in \MS \subset \BR^N$. For any $\eta, \delta,\beta > 0$, for $m \geq \Omega((N/\eta)^{2} \cdot \log(N/(\eta\delta)))$, with probability at least $1-\delta$ over an i.i.d.~collection $u_1, \ldots, u_m \sim \MP$, any $\beta$-barycentric spanner $u_{i_1}, \ldots, u_{i_N}$ of $\{ u_1, \ldots, u_m\}$ is also an $(\eta, \beta)$-distributional spanner for $\MP$.
\end{lemma}

  Finally, we shall need the following standard lemma, which is a consequence of the fact that the mapping $\softmax$ is $1/2$-Lipschitz from $\| \cdot \|_\infty$ to $\| \cdot \|_1$. 
  \begin{lemma}
    \label{lem:tv-logit-bound}
    Fix vectors $L, L' \in \BR^\Sigma$, and define distributions $P, P' \in \Delta(\Sigma)$ by $P = \softmax(L), P' = \softmax(L')$. Then 
    \begin{align}
\tvd{P}{P'} \leq \frac{1}{4} \cdot \| L-L' \|_\infty \nonumber.
    \end{align}
  \end{lemma}

\subsection{Settings of the parameters}
\label{sec:parameter-settings}
We suppose that we are given an oracle $\Ologit$ which is an $\epapx$-approximate oracle (\cref{def:approx-logit-oracle}) to a distribution $\model \in \Delta(\Sigma^T)$, for some alphabet $\Sigma$ and sequence length $T$. Our algorithm also takes as input parameters $d \in \BN$ (representing the rank) and $\alpha > 0$ representing the boundedness parameter (of \cref{def:bounded-logit-oracle}). Fix $\delta \in (0,1)$ representing the desired failure probability of the algorithm and $\vep \in (0,1)$ representing the desired error of the learned model. We define the following parameters:
\begin{itemize}
\item $\beta = 2$. %
\item $\gamma = 4K\epapx$. %
\item $\gamthres =\gamma \cdot \sqrt{C_{\ref{lem:approximate-elliptic-potential}} d \log(2 \beta K\alpha^2/\gamma)}$, where $C_{\ref{lem:approximate-elliptic-potential}}$ is the constant of \cref{lem:approximate-elliptic-potential}. 
\item $K = C_{K} \cdot Td \log(\beta d \alpha T/(\epapx\delta))$ for a sufficiently large constant $C_K$ (specified in \cref{lem:good-k}).
\item $n = C_n \cdot \log(TK/\delta) / \gamthres$ %
  for a sufficiently large constant $C_n$ (see \cref{lem:most-covered}). %
\item $\eta = \frac{\vep}{3T^2n}$.  
\end{itemize}

\subsection{Definitions for \cref{alg:learning-approx}}
\label{sec:alg-definitions}
    We restate the following program from \cref{alg:learning-approx}, defined by a step $t \in [T]$, symbols $y_1, \ldots, y_t \in \Sigma$, subsets $\hat \MF_s \subset \tilde \MF_s\subset \Sigma^{\leq T-s+1}$ for each $s \in [t+1]$, and vectors $L_{s,i} \in \BR^{\tilde\MF_s}$ for $i \in [\ds], s \in [t]$. Moreover, each vector $L_{s,i}$ is assumed to be given by $L_{s,i}(f) = \Lapx(h_{s,i}, f)$ for some $h_{s,i} \in \Sigma^{s-1}$. The program's variables are vectors $c_s \in \BR^\ds$, $s \in [t]$, and $\hat L_s \in \BR^{\hat\MF_s}$, $s \in [t+1]$:
    \begin{subequations}
  \begin{align}
\hat L_s(f) = \sum_{i=1}^\ds c_{s,i} \cdot L_{s,i}(f) \qquad & \forall f \in \hat\MF_s, \quad s \in [t]\label{eq:hatl-spanner}\\
   c_1 = e_1 \in \BR^\ds, \qquad   \| c_s \|_\infty \leq  \cnorm \qquad & \forall s \in [t]\label{eq:cnorm-bound}\\
    \hat L_{s+1}(f) = \sum_{i=1}^\ds c_{s,i} \cdot  L_{s,i}(y_{s}, f) \qquad & \forall f \in \hat\MF_{s+1}, s \in [t]\label{eq:hatl-next}.
  \end{align}
\end{subequations}
  For the purpose of analysis, it will often be useful for us to extend the vectors $L_{s,i}$ to be vectors $L_{s,i} \in \BR^{\Sigma^{\leq T-s+1}}$ by letting $L_{s,i}(f) = \Lapx(h_{s,i} \circ f)$; and $\hat L_{s+1}$ to a vector in $\BR^{\Sigma^{\leq T-s+1}}$ by defining it on all $f \in \Sigma^{\leq T-s+1}$ via \cref{eq:hatl-next}. This exactly mirrors \cref{line:extend-vecs} of \cref{alg:learning-approx}. %

  The below definition characterizes the set of sequences $y_{1:t}$ for which the above program is feasible:
  \begin{definition}
Fix any $t \in [T]$, as well as, for $s \in [t]$, a subset $\tilde \MF_s \subset \Sigma^{\leq T-s+1}$ and vectors $L_{s,1}, \ldots, L_{s,\ds}\in \BR^{\tilde\MF_s}$ (which should be interpreted as a distributional spanner as constructed in \cref{line:dist-spanner-algo} of \cref{alg:learning-approx}). We let $\feasSet_t(L_{1:t,1:\ds}) \subset \Sigma^t$ denote the set of $y_{1:t}$ for which the program in \cref{eq:hatl-spanner,eq:cnorm-bound,eq:hatl-next} is feasible. (Note that we have excluded the sets $\tilde \MF_s$ from the notation for brevity; we interpret them as being implied by the vectors $L_{s,i}$.)
  \end{definition}

  For the purpose of analyzing \cref{alg:learning-approx}, we first make the following definition which specifies the distribution of the vectors $c_{t,s,u}, \hat L_{t,s,u}$ constructed in \cref{line:solve-program-algo} of \cref{alg:learning-approx}. 
  \begin{definition}[Distribution $\MQ_t(L_{1:t,1:\ds})$]
    \label{def:qt}
      Fix any $t \in [T]$, as well as, for $s \in [t]$, vectors $L_{s,1}, \ldots, L_{s,\ds} \in \BR^{\tilde\MF_s}$. We define the following distribution $\MQ_t(L_{1:t,1:\ds})$ over sequences $y_{1:t+1} \in \Sigma^{t+1}$, vectors $c_s \in \BR^\ds$, $s \in [t]$, and vectors $ \hat L_s \in \BR^{\hat\MF_s}$, $s \in [t+1]$:
  \begin{itemize}
  \item Draw $y_{1:t+1} \sim \model$.
  \item Compute $\hat L_s \in \BR^{\hat\MF_s}$ ($s \in [t+1]$), $c_s \in \BR^\ds$ ($s \in [t]$) which solve the program in \cref{eq:hatl-spanner,eq:cnorm-bound,eq:hatl-next}. If the program is not feasible, choose $c_s = 0$, $\hat L_s = 0$ for all $s$. 
  \end{itemize}
  Note that the distribution of $(y_{1:t+1}, c_{1:t}, \hat L_{1:t+1} ) \sim \MQ_t(L_{1:t,1:\ds})$ is exactly the distribution of the variables $(y_{1:t+1,u}, c_{t,1:t,u}, \hat L_{t,1:t+1,u})$ constructed in \cref{line:solve-program-algo} of \cref{alg:learning-approx}, for each $u \in [n]$.
\end{definition}

Finally, we define the event $\Eoracle$ which we will show to characterize when the learning guarantee of \cref{alg:learning-approx} holds:
\begin{definition}\label{def:eoracle}
We let $\Eoracle$ denote the event that the execution trace of $\Ologit$ in \cref{alg:learning-approx} is $\epapx$-approximately $\alpha$-bounded of rank $d$ (per \cref{def:bounded-logit-oracle}). 
\end{definition}

\subsection{Supporting lemmas}

We let $\Espanner$ denote the event that the call to \DistSpanner on \cref{line:dist-spanner-algo} of \cref{alg:learning-approx} returns an $(\eta, \beta)$-distributional spanner for $\MP_t$ (\cref{def:dist-spanner}) for all epochs $k \in [K]$ and steps $t \in [T]$
\begin{lemma}
  \label{lem:espanner-prob}
  $\Espanner$ occurs with probability at least $1-\delta/2$.
\end{lemma}
\begin{proof}
\cref{lem:dist-exact-RN} together with the fact that \DistSpanner is called with parameters $\eta, \delta/(10KT)$ on \cref{line:dist-spanner-algo} ensures that each call to \DistSpanner returns a distributional spanner with probability at least $1-\delta/(10KT)$. %
  The claim follows by a union bound. 
\end{proof}

The below lemma shows that the program \cref{eq:hatl-l-close,eq:cnorm-bound-alg,eq:join-steps} in \cref{alg:learning-approx} is feasible. For the statement of the lemma, we recall the definition of the distribution $\MP_s \in \Delta(\BR^{\tilde \MF_s})$, for each $s$, as defined in \cref{line:define-pt} of \cref{alg:learning-approx}.
\begin{lemma}
  \label{lem:feasibility}
  Fix $t \in [T]$, and suppose we are given sets $\hat\MF_s \subset \tilde \MF_s \subset \Sigma^{\leq T-s+1}$ (for $s \in [t+1]$) and vectors $L_{1:t,1:\ds}$ as in \cref{alg:learning-approx}. Suppose that $L_{s,1:\ds}$ is an $(\eta, \beta)$-distributional spanner for the distribution of $L_s^\st \sim \MP_s$ for each $s \in [t]$. Then with probability $1-\eta \cdot t$ over the draw of $y_{1:t} \sim \model$, the program in  \cref{eq:hatl-spanner,eq:cnorm-bound,eq:hatl-next} is feasible. 
\end{lemma}
\begin{proof}
Recall that $\MP_s$ is defined by drawing $y_{1:s-1} \sim \model$ and writing $L_s^\st(f) = \Lapx(y_{1:s-1} \circ f)$ for $f \in \tilde \MF_s$. Also recall that $\tilde \MF_s = \hat \MF_s \cup (\Sigma\circ \hat\MF_{s+1})$. %
  
  For each $s \in [t]$, using the fact that $(L_{s,1}, \ldots, L_{s,\ds})$ is a $(\eta, \beta)$-distributional spanner for the distribution of $L_s^\st \sim \MP_s$, we have that with probability $1-\eta$ over the draw of $y_{1:s-1} \sim \model$, letting $L_s^\st$ be the random variable defined by $L_s^\st(\cdot) := \Lapx( y_{1:s-1},\cdot)$, there is $c_s \in \BR^\ds$ with $\| c_s \|_\infty \leq \beta$ so that %
  \begin{align}
     L_s^\st(f) = \sum_{i=1}^\ds c_{s,i} \cdot L_{s,i}(f)  \quad \forall f \in \hat\MF_s\label{eq:qbar-spanner}\\
L_s^\st(y_s', f) = \sum_{i=1}^\ds c_{s,i} \cdot L_{s,i}(y_s', f) \quad \forall y_s' \in \Sigma,\ f \in \hat\MF_{s+1} \label{eq:lsi-vec-ys}.
  \end{align}
  By a union bound over $s \in [t]$, we have that \cref{eq:qbar-spanner,eq:lsi-vec-ys} simultaneously hold for all $s \in [t]$ with probability at least $1-\eta t$ over the draw of $y_{1:t} \sim \model$. Let us condition on this event, and choose $c_s$, $s \in [t]$, so that \cref{eq:qbar-spanner,eq:lsi-vec-ys} hold for each $s \in [t]$. 
  
  For each $s \in [t]$ and $f \in \hat\MF_s$, define $\hat L_s(f) := L_s^\st(f)$. Then \cref{eq:qbar-spanner} verifies \cref{eq:hatl-spanner} and, noting that for $f \in \hat\MF_{s+1}$, we have $\hat L_{s+1}(f) = L_{s+1}^\st(f) = L_s^\st(y_s, f)$, \cref{eq:lsi-vec-ys} (with $y_s' = y_s$) verifies \cref{eq:hatl-next}. \cref{eq:cnorm-bound} evidently holds as well, completing the proof. (For the case $s=1$, we have that $L_1^\st = \Lapx(\cdot)$ with probability $1$, so we can take $L_{1,1} = \Lapx(\cdot)$, $c_{1,1} = 1$, and $c_{1,i} = 0$ for $i > 1$.)
\end{proof}

The below lemma shows that, if the number of epochs $K$ in \cref{alg:learning-approx} is large enough, then \emph{under the event $\Eoracle$} (\cref{def:eoracle}) and some additional high-probability events, there is some $k \in [K]$ for which in \cref{alg:learning-approx}, the \textbf{if} statement on \cref{line:test-discrepancy} evaluates to False for all $t \in [T]$. 
\begin{lemma}
  \label{lem:good-k}
Suppose $K \geq C \cdot Td \log(\beta d \alpha T/\epapx)$, for a sufficiently large constant $C$. Then there is some event $\ME_{\ref{lem:good-k}}$ in \cref{alg:learning-approx} which occurs with probability at least $1-\delta/4$ so that, under the event $\Eoracle \cap \Espanner \cap \ME_{\ref{lem:good-k}}$, there exists some epoch $k \in [K]$ for which the \textbf{if} statement in \cref{line:test-discrepancy} evaluates to False for all steps $t \in [T]$. 
\end{lemma}
\begin{proof}
  Note that each set $\hat\MF_s$ is initialized to be of size $1$, and is increased in size by at most $1$ at each epoch $k \in [K]$. Thus, 
  at each step of \cref{alg:learning-approx} when the \textbf{if} statement in \cref{line:test-discrepancy} is evaluated, we have $|\hat\MF_s| \leq K$. %
   In the event that the \textbf{if} statement evaluates to True at some epoch $k \in [K]$, for some $t \in [T]$ and choices of  $u \in [n]$, $s,r \in [t+1]$ with $r \geq s$, then we have, for some $y_r' \in \Sigma$: 
  \begin{align}
    \left| \hat L_{t,s,u}(y_{s:r-1,u},y_r') - \sum_{i=1}^\ds c_{t,s,u,i} \cdot L_{s,i}(y_{s:r-1,u},y_r') \right| >& \gamthres =   \gamma \cdot \sqrt{C_{\ref{lem:approximate-elliptic-potential}} d \log(2  \beta K \alpha^2/\gamma)}\label{eq:add-to-fs-1}\\
     \hat L_{t,s,u}(f) - \sum_{i=1}^\ds c_{t,s,u,i} \cdot L_{s,i}(f) = & 0 \qquad \forall f \in \hat\MF_s\label{eq:add-to-fs-2}. 
  \end{align}
  Moreover, in this event, then $\hat \MF_s$ is updated to be $\hat \MF_s \cup \{y_{s:r-1,u} \circ y_r'\}$. %

  By the construction of the distributional spanners $L_{1:t,1:\ds}$ (see \cref{line:dist-spanner-algo}), we have the following: for each $s \in [t+1]$, $i \in [\ds]$, and $f \in \Sigma^{\leq T-s+1}$, $L_{s,i}(f) = \Lapx(h_{s,i}\circ f)$ for some history $h_{s,i} \in \Sigma^{s-1}$. Moreover, by the way in which we have extended $\hat L_{t,s,u}(f), L_{s,i}(f)$ to be defined for all $f \in \Sigma^{\leq T-s+1}$ (\cref{line:extend-vecs}), %
  we have that, for all $f \in \Sigma^{T-s+1}$, 
  \begin{align}
    & \hat L_{t,s,u}(f) - \sum_{i=1}^\ds c_{t,s,u,i} \cdot L_{s,i}(f)\nonumber\\
    = & \sum_{i=1}^\ds c_{t,s-1,u,i} \cdot \Lapx(h_{s-1,i}\circ  y_{s-1,u}\circ f) - \sum_{i=1}^\ds c_{t,s,u,i} \cdot \Lapx(h_{s,i}\circ f)\label{eq:lhat-lsi-diff}
  \end{align}
  Next, the assumption that $\Eoracle$ holds, i.e., the execution trace of $\Ologit$ is  $\epapx$-approximately $\alpha$-bounded of rank $d$ (\cref{def:bounded-logit-oracle}) gives that there is some matrix $\logitmatrix' \in \BR^{\Sigma^{s-1} \times (\Sigma^{\leq T-s} \times \Sigma)}$ which is $\alpha$-bounded of rank $d$ (per \cref{def:bounded-logit}) so that, forx all $h \in \Sigma^{s-1}, f \in \Sigma^{\leq T-s}, y \in \Sigma$ for which the algorithm queries $\Ologit(h \circ f)$, we have $|\logitmatrix'(h, (f,y)) - \Lapx(h\circ f \circ y)| \leq \epapx$. Moreover, for all such $h,f,y$, we have $\logitmatrix'(h, (f,y)) = \langle \Bhist_h, \Bfut_{f,y} \rangle$ for some vectors $\Bhist_h, \Bfut_{f,y} \in \BR^d$ with $\| \Bhist_h \|_2 \leq \alpha$, $\| \Bfut_{f,y} \|_2 \leq \alpha$. (Note that $\Bfut_{f,y}$ may depend on $s$; we omit this from the notation for brevity.)

  Applying this for either (a) the future $(f', y) = (y_{s:r-1,u}, y_r')$ satisfying \cref{eq:add-to-fs-1,eq:add-to-fs-2} which is added to $\hat \MF_s$ at the present step or (b) any future $f = (f',y)$ which is already in $\hat \MF_s$ at the current step: we note that the  
  algorithm queries $\Ologit(h_{s-1,i} \circ  y_{s-1,u}\circ f')$ and $\Ologit(h_{s,i} \circ f')$ for each $i \in [\ds]$ (in order to evaluate the quantity on \cref{line:test-discrepancy}, in case (a); or else on \cref{line:dist-spanner-algo}, in case (b)). Thus, for such $(f',y)$, using \cref{eq:lhat-lsi-diff} with $f = (f',y)$, we get that 
  $\hat L_{t,s,u}(f) - \sum_{i=1}^\ds c_{t,s,u,i} \cdot L_{s,i}(f)$ is $2\ds \epapx$-close to %
  \begin{align}
    &  \sum_{i=1}^\ds c_{t,s-1,u,i} \cdot  \logitmatrix'(h_{s-1,i}\circ y_{s-1,u}, (f',y)) - \sum_{i=1}^\ds c_{t,s,u,i} \cdot \logitmatrix'(h_{s,i}, (f',y))\nonumber\\
    =& \left\lng \Bfut_{ (f',y)}, \sum_{i=1}^\ds c_{t,s-1,u,i} \cdot \Bhist_{(h_{s-1,i}\circ  y_{s-1,u})} - \sum_{i=1}^\ds c_{t,s,u,i} \cdot \Bhist_{h_{s,i}}\right\rng.\label{eq:ip-fut-hist}
  \end{align}
  Then the bound $\| c_{t,s,u} \|_\infty\leq \beta$ together with \cref{eq:ip-fut-hist} and the fact that $\ds \leq  K$ ensures that the collection of vectors $\hat L_{t,s,u} - \sum_{i=1}^\ds c_{t,s,u,i} \cdot L_{s,i} \in \BR^{\Sigma^{\leq T-s+1}}$ seen over the course of all updates to $\hat \MF_s$ is $(2\beta K \alpha,  \alpha, 2\ds\epapx)$-bounded of rank $d$, with respect to the sequence of elements of $\Sigma^{\leq T-s+1}$ corresponding to the new element added to $\hat \MF_s$ each time it is updated.\footnote{We recall that we consider $\hat L_{t,s,u} \in \BR^{\Sigma^{\leq T-s+1}}$ a vector by extending the definition \cref{eq:join-steps} to all $f \in \Sigma^{\leq T-s+1}$.}  Since $\epapx \leq \gamma/4\ds$ (see \cref{sec:parameter-settings}), it follows from \cref{lem:approximate-elliptic-potential} and \cref{eq:add-to-fs-1,eq:add-to-fs-2} that the number of steps at which the \textbf{if} statement on \cref{line:test-discrepancy} can evaluate to True for some value of $t$ is bounded above by $O\left(d \log(\beta d \alpha K / \epapx)\right)$.

  Under the event $\Espanner$, \cref{lem:feasibility} ensures that, at each epoch $k \in [K]$, with probability at least $1-\eta T^2 n \geq 1/2$, the program in \cref{eq:hatl-l-close,eq:cnorm-bound-alg,eq:join-steps} is feasible for each $t \in [T]$ and each sample $u \in [n]$. Moreover, if the program is feasible for each $t \in [T], u \in [n]$, then we always reach \cref{line:test-discrepancy} at all steps $t \in [T]$. Let $\ME_{\ref{lem:good-k}}$ be the event that for at least $K/3$ epochs $k \in [K]$, the program in \cref{eq:hatl-l-close,eq:cnorm-bound-alg,eq:join-steps} is feasible for each $t \in [T]$ and each sample $u \in [n]$ (or that the algorithm returns at some epoch $k < K$). As long as $K \geq \Omega(\log(1/\delta))$, we have that $\Pr(\ME_{\ref{lem:good-k}}) \geq 1-\delta/4$. 

  Thus, as long as $K \geq \Omega(Td \log(\beta d \alpha T /(\epapx\delta)))$,  under $\ME_{\ref{lem:good-k}} \cap \Espanner$, there  exists some epoch $k \in [K]$ for which the \textbf{if} statement on \cref{line:test-discrepancy} evaluates to False for all $t \in [T]$. 
\end{proof}

The next lemma shows that if, for some $k \in [K], t \in [T]$, the \textbf{if} statement in \cref{line:test-discrepancy} evaluates to false, then with high probability over $y_{1:t} \sim \model$, a certain inequality (i.e., \cref{eq:qt-exp-bound}) holds which is roughly the in-expectation version of the \textbf{if} statement in \cref{line:test-discrepancy}.
\begin{lemma}
  \label{lem:most-covered}
Suppose the parameter $n$ in \cref{alg:learning-approx} is chosen so that $n \geq C_n \cdot \log(TK/\delta) \cdot  (K\beta\alpha / \gamthres)$, for a sufficiently large constant $C_n$. 
  There is an event $\ME_{\ref{lem:most-covered}}$ that occurs  with probability at least $1-\delta/4$ over the evaluation of \cref{alg:learning-approx} so that the following holds under $\ME_{\ref{lem:most-covered}}$. 
 For any epoch $k \in [K]$ for which the \textbf{if} statement on \cref{line:test-discrepancy} evaluates to False for all $t \in [T]$, then letting $L_{1,1:\ds}, \ldots, L_{T,1:\ds}$ denote the spanners constructed in \cref{line:dist-spanner-algo} at epoch $k$, we have, for all $t \in [T]$, 
  \begin{align}
\E_{\MQ_t(L_{1:t,1:\ds})} \left[ \One{y_{1:t} \in \feasSet_t(L_{1:t,1:\ds})} \cdot \min \left\{ 1, \max_{\substack{s,r \in [t+1]\\r \geq s}} \max_{y_r' \in \Sigma} \left| \hat L_s(y_{s:r-1},y_r') - \sum_{i=1}^\ds c_{s,i} \cdot L_{s,i}(y_{s:r-1},y_r') \right|\right\} \right] \leq \gamthres. \label{eq:qt-exp-bound}
  \end{align}
\end{lemma}
\begin{proof}%
  Fix $t \in [T]$. Given vectors $L_{s,i} \in \BR^{\Sigma^{\leq T-s+1}}$, $s \in [t], i \in [\ds]$, recall the distribution $\MQ_t(L_{1:1:\ds})$ over $(y_{1:t+1}, c_{1:t}, \hat L_{1:t+1})$ from \cref{def:qt}. 
We define
  \begin{align}
q_0 := \E_{\MQ_t(L_{1:t,1:\ds})} \left[  \One{y_{1:t} \in \feasSet_t(L_{1:t,1:\ds})} \cdot \min \left\{ 1, \max_{\substack{s,r \in [t+1]\\r \geq s}} \max_{y_r' \in \Sigma} \left| \hat L_s(y_{s:r-1},y_r') - \sum_{i=1}^\ds c_{s,i} \cdot L_{s,i}(y_{s:r-1},y_r') \right|\right\} \right]\label{eq:q0-define}.
  \end{align}
    Suppose that we draw $n$ i.i.d.~samples $(y_{1:t+1,u}, c_{1:u}, \hat L_{1:t+1,u}) \sim \MQ_t(L_{1:1:\ds})$, for $u \in [n]$ (as in \cref{alg:learning-approx}). Then each sample $u \in [n]$ satisfies
  \begin{align}
    \label{eq:lsu-big}
 \One{y_{1:t} \in \feasSet_t(L_{1:t,1:\ds})} \cdot  \max_{\substack{s,r \in [t+1] \\ r \geq s}} \max_{y_r' \in \Sigma}   \left| \hat L_s(y_{s:r-1,u},y_r') - \sum_{i=1}^\ds c_{s,i,u} \cdot L_{s,i,u}(y_{s:r-1,u},y_r') \right| \geq q_0
  \end{align}
  with probability at least $q_0$. (This holds because the quantity inside the expectation in \cref{eq:q0-define} is always in $[0,1]$.) Then as long as $n \geq \Omega(\log(1/\delta) / q_0)$, there is an event $\ME$ that occurs with probability at least $1-\delta$ so that, under the event $\ME$, \cref{eq:lsu-big} holds for some $u \in [n]$.

Using that $\ds \leq K$, by our choice of $n \geq \Omega(\log(TK/\delta) / \gamthres)$, we see by a union bound over all $t \in [T]$ and epochs $k \in [K]$ in \cref{alg:learning-approx} that  for some event $\ME_{\ref{lem:most-covered}}$ occurring with probability $1-\delta$, for all epochs $k \in [K]$ and steps $t \in [T]$, either the quantity $q_0$ corresponding to episode $k$ and step $t$ satisfies $q_0 \leq \gamthres$, or else the \textbf{if} statement on \cref{line:test-discrepancy} evaluates to True. Thus, conditioned on the event that $\ME_{\ref{lem:most-covered}}$ holds and at some epoch $k$ the \textbf{if} statement on \cref{line:test-discrepancy} evaluates to False for all $t \in [T]$, we see that \cref{eq:qt-exp-bound} holds for all $t \in [T]$ for that epoch $k$. 
\end{proof}

Next, using \cref{lem:most-covered}, we show that ay any epoch when the \textbf{if} statement on \cref{line:test-discrepancy} evaluates to False for all $t \in [T]$, then the vectors $\hat L_{t,t+1}$ constructed in \cref{line:solve-program-algo} yield next-token logit values which are close to the true next-token logit values $\logit(y_{1:t}\circ y_{t+1}')$ in expectation over $y_{1:t} \sim \model$. To simplify notation, we state the lemma with respect to a draw of $(y_{1:t}, c_{1:t}, \hat L_{1:t+1}) \sim \MQ_T(L_{1:t,1:\ds})$, which is identical to the distribution of the variables constructed in \cref{line:solve-program-algo} of \cref{alg:learning-approx}.  Moreover, recall our convention (extending the constraint \cref{eq:hatl-next} to all $f$) that for such vectors $\hat L_{1:t+1}$, we write $\hat L_{s+1}(f) = \sum_{i=1}^\ds c_{s,i} \cdot L_{s,i}(y_s, f)$ for $f \in \Sigma^{\leq T-s}$. 

\begin{lemma}
  \label{lem:sampling-accuracy}
  Under the event $ \ME_{\ref{lem:most-covered}}$ of \cref{lem:most-covered}, the following holds. For any epoch $k \in [K]$ for which the \textbf{if} statement on \cref{line:test-discrepancy} evaluates to False for all $t \in [T]$, then letting $L_{1:T,1:\ds}$ denote the spanners constructed in \cref{line:dist-spanner-algo} in epoch $k$, we have, for each $t \in [T]$: 
    \begin{align}
\E_{\MQ_t(L_{1:t,1:\ds})}\left[  \One{y_{1:t} \in \feasSet_t(L_{1:t,1:\ds})} \cdot \min\left\{ 1, \max_{y_{t+1}' \in \Sigma} \left| \logit(y_{1:t}\circ y_{t+1}') - \hat L_{t+1}(y_{t+1}') \right| \right\}\right] \leq t \cdot \gamthres  +  \epapx\nonumber. %
  \end{align}
\end{lemma}
\begin{proof}
Fix $t \in [T]$, and consider a draw of $( y_{1:t}, c_{1:t}, \hat L_{1:t+1} ) \sim \MQ_t(L_{1:t,1:\ds})$. 
For $1 \leq s \leq t+1$, define the random variable $L_s^\star \in \BR^{\Sigma^{\leq T-s+1}}$ by $L_s^\star(f) := \Lapx( y_{1:s-1},f)$. This definition immediately gives $L_{s+1}^\st(f) = L_s^\st(y_s, f)$ for all $f$. 

Since we have conditioned on $ \ME_{\ref{lem:most-covered}}$, \cref{eq:qt-exp-bound} gives that, for each $s,r \in [t+1]$ with $r \geq s$,
  \begin{align}
\E_{\MQ_t(L_{1:t,1:\ds})}\left[ \One{y_{1:t} \in \feasSet_t(L_{1:t,1:\ds})} \cdot \min \left\{ 1,  \max_{y_r' \in \Sigma} \left| \hat L_s(y_{s:r-1}, y_r') - \left( \sum_{i=1}^\ds c_{s,i} \cdot L_{s,i}(y_{s:r-1},y_r')\right) \right| \right\}\right] \leq  \gamthres\label{eq:hatl-basis-close}.
  \end{align}

  For each $r,s$ satisfying $1 \leq s \leq r \leq t+1$, define
  \begin{align}
\vep_{s,r}(y_{1:t}) := \max_{y_r' \in \Sigma} \left| L_s^\st(y_{s:r-1},y_r') - \hat L_s(y_{s:r-1},y_r') \right|\nonumber.
  \end{align}
    Let us write $\BI_{y_{1:t}} := \One{y_{1:t} \in \feasSet_t(L_{1:t,1:\ds})}$.  For $r \geq s+1$, we have
  \begin{align}
&     \E_{\MQ_t(L_{1:t,1:\ds})}[\BI_{y_{1:t}} \cdot \min \{ 1, \vep_{s+1,r}(y_{1:t})\}] \nonumber\\
    =  & \E_{\MQ_t(L_{1:t,1:\ds})} \left[\BI_{y_{1:t}} \cdot \min \left\{ 1,  \max_{y_r' \in \Sigma} \left| L_{s+1}^\st(y_{s+1:r-1},y_r') - \hat L_{s+1}(y_{s+1:r-1},y_r') \right|\right\}\right]\nonumber\\
    =  & \E_{\MQ_t(L_{1:t,1:\ds})} \left[\BI_{y_{1:t}} \cdot \min\left\{ 1,  \max_{y_r' \in \Sigma} \left| L_{s}^\st(y_{s:r-1},y_r') - \hat L_{s+1}(y_{s+1:r-1},y_r') \right|\right\}\right]\nonumber\\
    \leq & \E_{\MQ_t(L_{1:t,1:\ds})} \left[\BI_{y_{1:t}} \cdot \min \left\{ 1,  \max_{y_r' \in \Sigma} \left| L_s^\st(y_{s:r-1},y_r')  - \hat L_s(y_{s:r-1},y_r')  \right|\right\}\right]  \nonumber\\
    & + \E_{\MQ_t(L_{1:t,1:\ds})}\left[\BI_{y_{1:t}} \cdot  \min \left\{ 1,  \max_{y_r' \in \Sigma} \left| \hat L_s(y_{s:r-1},y_r')  -  \sum_{i=1}^\ds c_{s,i} \cdot L_{s,i}(y_{s:r-1},y_r')  \right|\right\}\right]\nonumber\\
             \leq & \E_{\MQ_t(L_{1:t,1:\ds})}[\BI_{y_{1:t}} \cdot  \min\left\{ 1,  \vep_{s,r}(y_{1:t})\right\}] +  \gamthres\label{eq:vep-decrease}
  \end{align}
  where the final inequality uses \cref{eq:hatl-basis-close} and the definition of $\vep_{s,r}(y_{1:t})$. 
  Using the fact that $\sum_{i=1}^\ds c_{1,i} \cdot L_{1,i}(\cdot) = L_1^\st(\cdot) = \Lapx(\cdot)$, %
  a consequence of \cref{eq:hatl-basis-close} with $s=1$ is that for any $r \in [t+1]$, 
  \begin{align}
\E_{\MQ_t(L_{1:t,1:\ds})} \left[\BI_{y_{1:t}} \cdot \min\left\{ 1, \max_{y_r' \in \Sigma} \left|  \hat L_1(y_{1:r-1},y_r') -  L_1^\st(y_{1:r-1},y_r') \right|\right\} \right] \leq  \gamthres\label{eq:base-case}. 
  \end{align}

  Using \cref{eq:vep-decrease} for $r = t+1$ and $s \in \{1, 2, \ldots, t\}$ gives
  \begin{align}
&      \E_{\MQ_t(L_{1:t,1:\ds})}\left[\One{y_{1:t} \in \feasSet_t(L_{1:t, 1:\ds})} \cdot \min\left\{1, \max_{y_{t+1}' \in \Sigma} \left| L_{t+1}^\st(y_{t+1}') - \hat L_{t+1}(y_{t+1}') \right|\right\}\right]
    \nonumber\\
    = &  \E_{y_{1:t+1} \sim \model} [\BI_{y_{1:t}} \cdot \min\left\{ 1, \vep_{t+1, t+1}(y_{1:t})\right\}]\nonumber\\
    \leq & (t-1) \cdot \gamthres + \E_{\MQ_t(L_{1:t,1:\ds})} \left[ \BI_{y_{1:t}} \cdot \min\{ 1,  \vep_{1,r}(y_{1:t})\} \right] \leq t \cdot \gamthres,\nonumber
  \end{align}
  where the final inequality uses \cref{eq:base-case}. The conclusion of the lemma statement follows by noting that $L_{t+1}^\st(y_{t+1}') = \Lapx( y_{1:t},y_{t+1}')$ satisfies
  \begin{align}
  \E_{\MQ_t(L_{1:t,1:\ds})}\left[ \max_{y_{t+1}' \in \Sigma} \left| \Lapx( y_{1:t},y_{t+1}') - \logit( y_{1:t}\circ y_{t+1}') \right| \right] \leq \epapx\label{eq:use-apx-oracle},
  \end{align}
  since we have assumed $\Ologit$ (from which $\Lapx(y_{1:t}, y_{t+1}')$ is derived) is an $\epapx$-approximate logit oracle (\cref{def:approx-logit-oracle}).
\end{proof}

The next lemma establishes the straightforward fact that if some procedure $G_t : \Sigma^{t-1} \to \Delta(\Sigma)$ (e.g., the single-step sampling procedure of \cref{alg:learning-approx}) approximates the true next-token distribution $\model(y_t = \cdot \mid y_{1:t-1})$ at each step $t$, then iterating this step approximates $\model$ in total variation distance.
\begin{lemma}
  \label{lem:coupling}
  Suppose that for each $t \in [T]$, $G_t : \Sigma^{t-1} \to \Delta(\Sigma)$ is a deterministic mapping which satisfies
  \begin{align}
\E_{y_{1:t-1} \sim \model}\left[ \tvd{G_t(y_{1:t-1})}{\model(y_t = \cdot \mid y_{1:t-1})}\right] \leq \vep_t\nonumber,
  \end{align}
  for some choices of $\vep_t \geq 0$. Consider the distribution $\model_G \in \Delta(\Sigma^T)$ which samples sequentially $y_t \sim G_t(y_{1:t-1})$, for $1 \leq t \leq T$. Then $\tvd{\model}{\model_G} \leq \sum_{t=1}^T \vep_t$. 
\end{lemma}
\begin{proof}
  We sequentially construct a coupling $(y_{1:T}, y'_{1:T})$ between $\model, \model_G$ as follows (we will have $y_{1:T} \sim \model, y_{1:T}' \sim \model_G$).  Given $y_{1:t-1}', y_{1:t-1}$, we consider two cases:
  \begin{itemize}
  \item If $y_{1:t-1}' = y_{1:t-1}$, then we define $y_t, y_t'$ to be distributed according to an optimal coupling between $G_t(y_{1:t-1}), \model(y_t = \cdot \mid y_{1:t-1})$. In particular, this ensures that
    \begin{align}
\Pr\left( y_t \neq y_t' \mid y_{1:t-1}, y'_{1:t-1}\right) \leq \vep_t\nonumber.
    \end{align}
  \item If $y_{1:t-1}' \neq y_{1:t-1}$, then we sample $y_t \sim \model(y_t = \cdot \mid y_{1:t-1}), y_t' \sim G_t(y_{1:t-1})$ independently. 
  \end{itemize}
  It is clear that this procedure ensures $y_{1:T} \sim \model, y'_{1:T} \sim \model_G$. Moreover,
  \begin{align}
    \Pr(y_{1:T} \neq y'_{1:T}) \leq & \sum_{t=1}^T \Pr(y_t \neq y_t', y_{1:t-1} = y'_{1:t-1}) \nonumber\\
    \leq & \sum_{t=1}^T \E_{y_{1:t-1}, y'_{1:t-1}} \left[ \One{y_{1:t-1} = y'_{1:t-1}} \cdot \Pr(y_t \neq y_t' \mid y_{1:t-1}, y'_{1:t-1}) \right]\leq \sum_{t=1}^T \vep_t\nonumber,
  \end{align}
  as desired.
\end{proof}

The below lemma characterizes the queries made by \cref{alg:learning-approx} to the sampling oracle $\Ologit$. 
\begin{lemma}
  \label{lem:queries}
\cref{alg:learning-approx} can be implemented using $\tilde O \left(K^4T|\Sigma| \log(1/\delta)/\eta^2 + nT^3K^2|\Sigma|\right)$ queries to $\model$ and $\Ologit$. 
\end{lemma}
\begin{proof}
  Queries to the sampling oracle $\model$ and the logit oracle $\Ologit$ are made at the following locations in \cref{alg:learning-approx}:
  \begin{itemize}
  \item First, with each call to $\DistSpanner$ in \cref{line:dist-spanner-algo}, we draw $m = O((K/\eta)^2 \cdot \log(K/(\delta \eta)))$ histories $h_1, \ldots, h_m \sim \model_{\leq t}$, and compute $\Lapx(h_i \circ f)$ using $\Ologit$ (which requires 1 query per future, in order to do mean-centering of the logits) for each $i \in [m]$ and $f \in \tilde \MF_t$. Overall, since at most a single element is added to the support of the sets $\hat\MF_t$ each iteration, we have that $|\tilde \MF_t| \leq O(K|\Sigma|)$. Thus, this step makes a total of $\tilde O(K^4 |\Sigma| T \log(1/\delta) / \eta^2)$ calls to $\Ologit$ and $\model$.
  \item Next, at each of the (at most) $K$ times \cref{line:test-discrepancy} is called, we need to compute the elements $\hat L_{t,s,u}(y_{s:r-1,u}, y_r'), L_{s,i}(y_{s:r-1,u}, y_r')$ for all $s,r \in [t+1], u \in [n], i \in [d]$, which overall requires $O(nT^2|\Sigma|\ds)$ calls to $\Ologit$. Using that $\ds \leq O(K)$, overall this results in $O(nT^3K^2|\Sigma|)$ queries to $\Ologit$.
  \item \cref{line:sample-ytu} requires a total of $O(TnK)$ queries to a sampling oracle for $\model$. 
  \item \cref{line:final-logits} requires a total of $O(T|\Sigma| \ds) \leq O(T|\Sigma|K)$ queries to $\Ologit$.
  \end{itemize}
  Altogether, the number of queries made to $\model, \Ologit$ is bounded above by
  \begin{align}
\tilde O \left( K^4|\Sigma| T \log(1/\delta)/\eta^2 + nT^3K^2|\Sigma| + TnK\right) \leq & \tilde O \left(  K^4T|\Sigma| \log(1/\delta)/\eta^2 + nT^3K^2|\Sigma|\right)\nonumber.
  \end{align}
\end{proof}

The below theorem states our main learning guarantee under the assumption that we are given an $\epapx$-approximate oracle $\Ologit$ to $\model$, as well as the ability to sample from $\model$ directly. The theorem shows that on the event $\Eoracle$ that the execution trace of $\Ologit$ is approximately low-rank and bounded (\cref{def:eoracle}), then \cref{alg:learning-approx} and \cref{alg:sampling-approx} together yield a distribution $\hat \model$ which approximates $\model$ in total variation distance.
\begin{theorem}
  \label{thm:infty-oracle-main}
  We suppose that we are given an oracle $\Ologit$ which is an $\epapx$-approximate oracle (\cref{def:approx-logit-oracle}) to a distribution $\model \in \Delta(\Sigma^T)$. Fix $\delta, \vep \in (0,1)$ satisfying $\vep \geq \Omega\left(\epapx \cdot K T^2 d^{1/2} \log^{1/2}(K\alpha/\epapx)\right)$. Then there is some event $\ME$ which occurs with probability $1-\delta$ so that under $\ME \cap \Eoracle$, \cref{alg:learning-approx} outputs data $(L_{1:T,1:d}, \hat\MF_{1:T}, \tilde \MF_{1:T})$ which, when passed to \cref{alg:sampling-approx}, lead the algorithm to produce samples from some distribution $\hat \model$ satisfying $\tvd{\model}{\hat \model} \leq \vep$. 
  Moreover, the number of queries to $\Ologit, \model$ made by \cref{alg:learning-approx} is bounded above by $N = \tilde O \left( \frac{T^{13} d^4 |\Sigma| \poly\log(\alpha/\delta)}{\vep^4} \right)$, %
  and the algorithm runs in $\poly(N)$ time. 
\end{theorem}
We remark that the guarantee of \cref{thm:infty-oracle-main} holds unchanged if the oracle $\Ologit$ satisfies the following weaker approximation condition:
\begin{align}
  \E_{y_{1:t} \sim \model}\left[  \left\| \projOne \Ologit(y_{1:t}) - \logit(y_{t+1} = \cdot \mid y_{1:t}) \right\|_\infty \right] \leq \epapx\label{eq:weaker-ologit-oracle}.
\end{align}
Indeed, the only point at which we need to use that $\Ologit$ is an $\epapx$-approximate oracle is in the proof of \cref{lem:sampling-accuracy}, where we use this condition to establish \cref{eq:use-apx-oracle}, which is evidently also a consequence of \cref{eq:weaker-ologit-oracle}. 
\begin{proof}[Proof of \cref{thm:infty-oracle-main}]
  By \cref{lem:good-k}, under the event $\Espanner \cap \Eoracle \cap \ME_{\ref{lem:good-k}}$, there is some  epoch $k$ for which the \textbf{if} statement on \cref{line:test-discrepancy} evaluates to False for all $t \in [T]$; the distributional spanners constructed in that epoch, which we denote by $L_{1:T, 1:\ds}$, are the ones returned by \cref{alg:learning-approx}.
  By \cref{lem:sampling-accuracy}, it follows that, under  the event $\Eoracle \cap \Espanner \cap \ME_{\ref{lem:good-k}} \cap \ME_{\ref{lem:most-covered}}$, for all $t \in [T]$, we have
  \begin{align}
\E_{\MQ_t(L_{1:t,1:\ds})}\left[ \One{y_{1:t} \in \feasSet_t(L_{1:t,1:\ds})} \cdot\min\left\{ 1, \max_{y_{t+1}'\in\Sigma}  \left| \logit(y_{1:t}\circ y_{t+1}') - \hat L_{t+1}(y_{t+1}') \right|\right\} \right] \leq t \cdot \gamthres + \epapx \nonumber.
  \end{align}
  We let $\hat L_{t+1}|_\Sigma \in \BR^\Sigma$ be the vector defined by $\hat L_{t+1}|_\Sigma(y) := \hat L_{t+1}(y)$ for $y \in \Sigma$. 
  Using the fact that %
  $\model(y_{t+1} = \cdot \mid y_{t+1}) = \softmax(\logit(y_{t+1} = \cdot \mid y_{1:t}))$, and the fact that $\tvd{\softmax(L)}{\softmax(L')} \leq \min \{ 1, \| L-L'\|_\infty\}$ for any $L, L' \in \BR^\Sigma$ (\cref{lem:tv-logit-bound}), we see from the above display that 
    \begin{align}
      & \E_{\MQ_t(L_{1:t,1:\ds})}\left[ \One{y_{1:t} \in \feasSet_t(L_{1:t,1:\ds})} \cdot\tvd{\model(y_{t+1} = \cdot \mid y_{1:t})}{\softmax(\hat L_{t+1}|_\Sigma)}  \right] \nonumber\\
      \leq &  t \cdot \gamthres +  \epapx\label{eq:feasible-tvd}.
    \end{align} 
    By \cref{lem:feasibility}, under the event $\Espanner$ (i.e., under which  $L_{t,1:\ds}$ is an $(\eta, \beta)$-distributional spanner for the distribution of $L_t^\st \sim \MP_t$ for each $t$), we have that, for each $t \in [T]$, 
    \begin{align}
\E_{\MQ_t(L_{1:t,1:\ds})} \left[ 1 - \One{y_{1:t} \in \feasSet_t(L_{1:t,1:\ds})} \right] \leq \eta \cdot t \label{eq:not-feasible}.
    \end{align}
    Combining \cref{eq:feasible-tvd,eq:not-feasible}, we see that, under $\Eoracle \cap \Espanner \cap \ME_{\ref{lem:good-k}} \cap \ME_{\ref{lem:most-covered}}$,
    \begin{align}
\E_{\MQ_t(L_{1:t,1:\ds})}\left[ \tvd{\model(y_{t+1} = \cdot \mid y_{1:t})}{\softmax(\hat L_{t+1}|_{\Sigma})}  \right] \leq t \cdot \gamthres +  \epapx + \eta\cdot t. \label{eq:tvd-sampling}
    \end{align}
    
  For $t \in [T]$, define $G_t : \Sigma^{t-1} \to \Delta(\Sigma)$ to be the following deterministic mapping:
  \begin{itemize}
  \item $G_1$ outputs the distribution $\hat P_1 \in \Delta(\Sigma)$ defined by $\hat P_1(y) \propto \softmax(\Lapx(y))$, for $y \in \Sigma$.
  \item For $t \geq 1$ and $y_{1:t} \in \Sigma^{t}$, $G_{t+1}(y_{1:t})$ is the distribution $\hat P_{t+1} \in \Delta(\Sigma)$ defined as follows: we first compute a solution $(c_{1:t}, \hat L_{1:t+1})$ to the program in \cref{eq:hatl-spanner,eq:cnorm-bound,eq:hatl-next} for step $t$, and set $\hat P_{t+1} := \softmax(\hat L_{t+1}|_\Sigma)$. 
  \end{itemize}
  Note that, for each $t$, the marginal  distribution of $(y_{1:t}, \hat L_{1:t+1}) \sim \MQ_t(L_{1:t,1:\ds})$ satisfies that $\softmax(\hat L_{t+1}|_\Sigma) = G_{t+1}(y_{1:t})$ and that    the marginal distribution of $y_{1:t} \sim \MQ_t(L_{1:t,1:\ds})$ is exactly that of $\model$. Thus, we can rewrite \cref{eq:tvd-sampling} as
      \begin{align}
\E_{y_{1:t} \sim \model}\left[ \tvd{\softmax(\model(y_{t+1} = \cdot \mid y_{1:t}))}{G_{t+1}(y_{1:t})}  \right] \leq t \cdot \gamthres +  \epapx + \eta t\label{eq:tvd-sampling-2}
    \end{align}
  
  Note that the distribution over $y_{1:T}$ produced by \cref{alg:sampling-approx} is equivalent to the distribution $\model_G$ introduced in \cref{lem:coupling}, namely that which, for each $t$, samples $y_t \sim G_t(y_{1:t-1})$. 
  Thus, it follows from \cref{lem:coupling} that, under the event $\Eoracle \cap \Espanner \cap \ME_{\ref{lem:good-k}} \cap \ME_{\ref{lem:most-covered}}$, $\tvd{\model}{\model_G} \leq T^2 \cdot (\gamthres + \eta) + T \cdot \epapx \leq \vep$, where the final inequality follows from the fact that $\eta \cdot T^2 \leq \vep/4$ and our assumption that 
  \begin{align}
T^2 \cdot \gamthres + T \cdot \epapx \leq C \cdot \epapx \cdot K T^2 d^{1/2} \log^{1/2}(K\alpha/\epapx) \leq \vep/2\label{eq:epapx-ep}
  \end{align}
  for a sufficiently large constant $C$ as well as the parameter settings in \cref{sec:parameter-settings}. Finally, by \cref{lem:good-k,lem:most-covered,lem:espanner-prob}, $\Pr(\Espanner \cap \ME_{\ref{lem:good-k}} \cap \ME_{\ref{lem:most-covered}}) \geq 1-\delta$ so we may take $\ME = \ME_{\ref{lem:good-k}} \cap \ME_{\ref{lem:most-covered}} \cap \Espanner$.

  Finally, to bound the number of queries to $\model, \Ologit$, we appeal to \cref{lem:queries}, which bounds the number of queries by $\tilde O \left( K^4T|\Sigma| \log(1/\delta)/\eta^2 + nT^3K^2|\Sigma|\right)$. To simplify this expression, we note that we may take $\epapx$ as large as possible so that the final inequality in \cref{eq:epapx-ep} holds. Doing so ensures that $\gamthres = \Theta(\vep / T^2)$. Then we may compute
  \begin{align}\label{eq:sample-ub}
    nT^3 K^2 |\Sigma| \leq  &   K^4 T |\Sigma| \log(1/\delta)/\eta^2 \leq \tilde O\left( T^5 d^4 |\Sigma| \poly\log(\alpha/\delta) \cdot \frac{T^{8}}{\vep^4}\right)\nonumber\\
    \leq &  \tilde O \left( \frac{T^{13} d^4 |\Sigma| \poly\log(\alpha/\delta)}{\vep^4} \right).
  \end{align}
    \end{proof}

\subsection{Proof of main theorem: learning an approximately low-logit rank model}
\label{sec:avg-apx-proof}
In this section, we extend the guarantee of \cref{thm:infty-oracle-main} to the setting where we only assume that $\model$ is close to a low-rank model in the sense of \cref{def:avg-closeness-true}, i.e., the \emph{average} distance between elements of their logit matrices is at most $\epavg$ for some sufficiently small $\epavg$. To establish this fact, we first prove that the queries that \cref{alg:learning-approx} makes to the logit oracle $\Ologit$ may be coupled to a set of polynomially many draws of samples from $\model$. 
\begin{lemma}
  \label{lem:coupling-indep-draws}
  It is possible to couple the execution of \cref{alg:learning-approx} with the draw of
  \begin{align}
    N := O \left( Kn + KT (K/\eta)^2 \log(K/(\eta \delta))\right)\nonumber
  \end{align}
  i.i.d.~samples $y\^1_{1:T}, \ldots, y\^N_{1:T} \sim \model$ so that each query to $\Ologit$ in \cref{alg:learning-approx}, is of the form $\Ologit(y\^j_{1:s-1}\circ  y_s'\circ y\^i_{s+1:t-1})$ for some $i,j \in [N]$, $0 \leq s < t \leq T$ and $y_s' \in \Sigma$. 
\end{lemma}
\begin{proof}
  Let us draw i.i.d.~samples $y\^i_{1:T} \sim \model$, $i \in [N]$. We couple the execution of \cref{alg:learning-approx} to the draw of these samples, as follows:
  \begin{itemize}
    \item We use $1$ of the samples $y\^i_{1:T}$ to implement the draw of $y_{1:T} \sim \model$ on \cref{line:initialize-ft}. 
  \item We use $Kn$ of the samples $y\^i_{1:T}$ to implement the draws of $y_{1:T,u} \sim \model$ on \cref{line:sample-ytu} for each epoch $k$ and each $u \in [n]$.
  \item In each call to \DistSpanner (\cref{alg:dist-spanner-simple}) on \cref{line:dist-spanner-algo}, we use a fresh set of $O((\ds/\eta)^2 \log(\ds/(\eta \delta))) \leq O((K/\eta)^2 \log(K/(\eta\delta)))$ of the samples $y\^i_{1:T}$ to implement \cref{line:draw-spanner-sample} of \cref{alg:dist-spanner-simple}. In particular, if we are calling \DistSpanner for the distribution $\MP_t \in \Delta(\BR^{\tilde \MF_t})$, then for each sample $y\^i_{1:T}$, the corresponding vector used in \cref{alg:dist-spanner-simple} is $\Lapx( y\^i_{1:t-1} \circ \tilde \MF_t)$. 
  \end{itemize}
  Since the samples $y\^i_{1:T}$ are drawn i.i.d.~it is straightforward to see that the above yields a coupling between the execution of \cref{alg:learning-approx} and the draws of $y\^i_{1:T}$.

Before proceeding, we  need the following claim:
  \begin{claim}
    \label{clm:supp-tildeqt}
At all epochs $k$ in \cref{alg:learning-approx}, for each $t \in [T]$, each element in $\tilde \MF_t$ may be written as $(y_t'\circ y_{t+1:r-1}\^i\circ y_r')$ for some $r \geq t$, some $i \in [N]$, and some $y_t', y_r' \in \Sigma$.
\end{claim}
\begin{proof}[Proof of \cref{clm:supp-tildeqt}]
The only sequences ever added to $\hat \MF_t$ (on \cref{line:test-discrepancy}) are of the form $y_{t:r-1,u}\circ y_r'$, which (by our coupling) may be written as $y\^i_{t:r-1}\circ y_r'$, for some $y_r' \in \Sigma$, $i \in [N]$, and $r \geq t$. The conclusion of the claim follows by the definition of $\tilde \MF_t$ on \cref{line:define-tildeqt}. 
\end{proof}

  Next we verify that the queries to $\Ologit$ in \cref{alg:learning-approx} are of the form claimed in the lemma statement.  In particular, we check this fact for each of the places where \cref{alg:learning-approx} makes queries to $\Ologit$ (recall that the algorithm makes queries to $\Ologit(y_{1:t-1})$ to obtain the value of the mean-centered logit $\Lapx(y_{1:t})$, for any $y_{1:t} \in \Sigma^t$):
  \begin{itemize}
  \item In each call to \DistSpanner on \cref{line:dist-spanner-algo}, per the coupling as discussed above, we need to evaluate $\Lapx( y_{1:t-1}\^i \circ f)$ for each $f \in \tilde \MF_t$. Using \cref{clm:supp-tildeqt}, each such $f \in \tilde \MF_t$ may be written as $f = (y_t'\circ y_{t+1:r-1}\^j\circ y_r')$ for some $j \in [N]$. Thus, the value of $\Lapx( y_{1:t-1}\^i\circ f)$ is computed as $\Lapx(y_{1:t-1}\^i\circ f) = \Lapx(y_{1:t-1}\^i\circ y_t'\circ y_{t+1:r-1}\^j\circ y_r')$. 
  \item In \cref{line:test-discrepancy}, for each $u \in [n],\ s,r \in [t+1]$, letting $f = (y_{s:r-1,u}, y_r')$, we need to evaluate
    \begin{align}
 \hat L_{t,s,u}(f) - \sum_{i=1}^d c_{t,s,u,i} \cdot L_{s,i}(f) = \sum_{i=1}^d c_{t,s-1,u,i} \cdot L_{s-1,i}(y_{s-1,u}\circ f) - \sum_{i=1}^d c_{t,s,u,i} \cdot L_{s,i}(f)\nonumber.
    \end{align}
    Thus, we need to evaluate $L_{s-1,i}(y_{s-1,u}, f), L_{s,i}(f)$ for each $i \in [d]$. Recall the definition of the histories $h_{t,i}$ inducing $L_{t,i}$ in \cref{line:dist-spanner-algo}. Fixing $s,i$, we can write $h_{s-1,i} = y_{1:s-2}\^j$ and $h_{s,i} = y_{1:s-1}\^\ell$ for some $j,\ell \in [N]$ (this holds because \DistSpanner (\cref{alg:dist-spanner-simple}) simply returns a subset of the vectors drawn in \cref{line:draw-spanner-sample}). 
    We can write $y_{s-1:r-1,u} = y_{s-1:r-1}\^k$ for some $k \in [N]$. Then
    \begin{align}
      L_{s-1,i}(y_{s-1,u}\circ  f) =&  \Lapx(y_{1:s-2}\^j\circ y_{s-1,u}\circ y_{s:r-1,u}\circ y_r') = \Lapx(y_{1:s-2}\^j\circ y_{s-1:r-1}\^k\circ y_r')\nonumber\\
      L_{s,i}(f) =& \Lapx(y_{1:s-1}\^\ell\circ y_{s:r-1,u}\circ y_r') = \Lapx(y_{1:s-1}\^\ell\circ y_{s:r-1}\^k\circ y_r')\nonumber, 
    \end{align}
    as desired. 
    \item Finally, it is straightforward to see in a similar manner that the calls to $\Ologit$ on \cref{line:final-logits} are also of the desired form. 
  \end{itemize}
\end{proof}

Finally, we are ready to state and prove our main theorem, which states that we can efficiently learn  a distribution $\model$ whose logit matrix is approximately low-rank ``on average over its entries'', in the sense of \cref{def:avg-closeness-true}.
\begin{theorem}
  \label{thm:approx-main}
  Fix $\vep^\st, \delta^\st \in (0,1)$. 
Suppose that $\model \in \Delta(\Sigma^{T})$ is of $\epavg$-approximate $\alpha$-bounded logit rank $d$ (\cref{def:avg-closeness-true}) for some $\epavg$ satisfying $\vep^\st \geq \tilde \Omega \left( \left( T^{31} d^{9.5} |\Sigma|^4 \poly\log(\alpha/\delta) \cdot \epavg\right)^{1/9} \right)$. %
Let $\Ologit$ be an exact logit oracle for $\model$ (\cref{def:exact-logit-oracle}).  Then for an appropriate choice of parameters, when \cref{alg:learning-approx} is run with the distribution $\model$ and the logit oracle $\Ologit$, its output distribution $\hat \model$ satisfies $\tvd{\hat\model}{\model} \leq \ep^\st$ with probability $1-\delta^\st$. It makes $N = \tilde O \left( \frac{T^{13} d^{4} |\Sigma|^4 \poly\log(\alpha/\delta^\st)}{(\vep^\st)^4} \right)$ queries to $\Ologit$ and runs in $\poly(N)$ time.
\end{theorem}
\begin{proof}
  Let us denote the mean-centered logits of $\model$ by $\logit := \logit_{\model}$. We continue on using the notation $\logit(y_{1:t}) = \logit(y_t\mid y_{1:t-1})$. %
  Fix $\vep^\st, \delta^\st \in (0,1)$. Define
  \begin{align}
  \vep = \vep^\st/2 , \qquad  \epapx = \frac{\vep}{C \cdot KT^2 d^{1/2}  \log^{1/2}( K\alpha/\vep)}, \qquad \vep = \vep^\st, \qquad \delta = \delta^\st / 2\nonumber,
  \end{align}
  where $C$ is a sufficiently large constant so as to ensure that the lower bound on $\vep$ in \cref{thm:infty-oracle-main} holds. Consider the execution of \cref{alg:learning-approx} with parameters $\beta, \gamma, \gamthres, K, n, \eta$ as specified in \cref{sec:parameter-settings} given the above values of $\epapx, \vep, \delta$. 
  Let $N := O(Kn + KT(K/\eta)^2 \log(K/(\eta \delta)))$ be the number of samples specified in \cref{lem:coupling-indep-draws}. 
  We consider the coupling of the execution of \cref{alg:learning-approx} to the draws of $y_{1:T}\^1, \ldots, y_{1:T}\^N \sim \model$ as in \cref{lem:coupling-indep-draws}, so that the only logit entries of $\logit$ used to implement $\Ologit$ in the oracle calls made by \cref{alg:learning-approx} are of the form $\Ologit(y_{1:s-1}\^i\circ y_s' \circ y_{s+1:t-1}\^j \circ y_t')$ for some $i,j \in [N]$, $0 \leq s < t \leq T$, and $y_s',y_t' \in \Sigma$.

  For any fixed values of $i,j \in [N]$, $0 \leq s < t \leq T$, and $y_t', y_s' \in \Sigma$, we have from the fact that $\model$ is $\epavg$-approximate $\alpha$-bounded rank $d$ (\cref{def:avg-closeness-weak}) that for any $\Delta > 0$, 
  \begin{align}
\Pr_{y_{1:T}\^i, y_{1:T}\^j \sim \model} \left[ \left| \logit( y_{1:s-1}\^i\circ y_s'\circ y_{s+1:t-1}\^j\circ y_t') - \logitmatrix\^s(y_{1:s-1}\^i\circ y_s',( y_{s+1:t-1}\^j, y_t')) \right| > \epavg \cdot \Delta |\Sigma|^2 \right] \leq 1/\Delta\nonumber,
  \end{align}
  where $\logitmatrix\^s \in \BR^{\Sigma^s \times (\Sigma^{\leq T-s-1} \times \Sigma)}$ is $\alpha$-bounded of rank $d$. 
  Thus, by a union bound, with probability at least $1-\delta^\st/4$, we have %
  \begin{align}
    \left| \logit( y_{1:s-1}\^i\circ y_s'\circ y_{s+1:t-1}\^j\circ y_t') - \logitmatrix\^s(y_{1:s-1}\^i\circ y_s',( y_{s+1:t-1}\^j, y_t')) \right| \leq \epavg \cdot T^2 N^2 |\Sigma|^2 \cdot \frac{4}{\delta^\st} \nonumber\\
    \forall i,j \in [N], \ 0 \leq s < t \leq T, \ y_t', y_s' \in \Sigma.\label{eq:union-coupling}
  \end{align}
  Let $\ME$ be the event that \cref{eq:union-coupling} holds, so that $\Pr(\ME) \geq 1  - \delta^\st/4$. Under the event $\ME$, the execution trace of \cref{alg:learning-approx} with respect to the logit oracle $\Ologit$ is $\epavg \cdot T^2 N^2 |\Sigma|^2 \cdot \frac{4}{\delta^\st}$-approximatly $\alpha$-bounded of rank $d$, i.e., $\ME \subset \Eoracle$. %
    The parameter settings in \cref{sec:parameter-settings} give that $N = \tilde O \left( \frac{\poly\log(\alpha/\delta) \cdot d^4 T^{13} |\Sigma|}{\vep^4} \right)$ (see the computation in \cref{eq:sample-ub}), and thus our assumption on $\epavg$ in the lemma statement gives that
  \begin{align}
\epavg \cdot T^2 N^2 |\Sigma|^2 \cdot \frac{4}{\delta^\st} \leq \frac{\vep}{C \cdot KT^2 d^{1/2} \log^{1/2}(K\alpha/\vep) } = \epapx\nonumber,
  \end{align}
  which means that $\Ologit$ and the error parameter $\epapx$ satisfy the assumptions of \cref{thm:infty-oracle-main} with error parameter $\epapx$ (in particular, recall that we have assumed for simplicity that $\Ologit$ is an exact logit oracle).\footnote{It is evidence that the proof goes through unchanged if instead we assume that $\Ologit$ is a ``sup-norm'' $\epapx$-approximate oracle in the sense that $\| \projOne\Ologit(y_{1:t}) - \logit(y_{t+1}= \cdot \mid y_{1:t}) \|_\infty \leq \epapx$ for all $y_{1:t}$.}  %

  Thus, we have from \cref{thm:infty-oracle-main} that on some event $\ME'$ occurring with probability $1-\delta$, \cref{alg:learning-approx} (given the distribution $ \model$ and the oracle $\Ologit$) produces some distribution $\hat \model$ satisfying $\tvd{\model}{\hat \model} \leq \vep$. Altogether, $\ME \cap \ME'$ occurs with probability at least $1-\delta^\st$, and thus with probability at least $1-\delta^\st$, the output $\hat \model$ of \cref{alg:learning-approx} when given $\model$ and the logit oracle $\Ologit$ satisfies $\tvd{\model}{\hat\model} \leq \ep^\star$. 
  
\end{proof}

\begin{remark}[Approximate logit oracle]
  \label{rmk:approx-logit-oracle}
It is straightforward to see that the guarantee of \cref{thm:approx-main} still holds if $\Ologit$ is only assumed to be an $\epavg$-approximate logit oracle (\cref{def:approx-logit-oracle}), with the loss of only constant factors in the bounds; indeed, we simply lose a constant factor in \cref{eq:union-coupling}.
\end{remark}

One natural follow-up question is whether a similar guarantee to \cref{thm:approx-main} holds if we are only able to make \emph{conditional sampling} queries to $\model$, a weaker model than having logit query access. In particular, a \emph{conditional sampling} oracle $\Osamp$ takes as input a sequence $y_{1:t}$ and outputs a draw $y_{t+1} \sim \model(y_{t+1} = \cdot \mid y_{1:t})$. Below we observe a few conditions under which such an oracle $\Osamp$ can be used to implement an $\ep$-approximate logit query oracle:
\begin{remark}[From conditional sampling oracle to logit oracle]
\label{rmk:conditional-sampling}
Suppose that the mean-centered logits are bounded as $|\logit_\model(y_{1:t})| \leq \lambda$ for all $y_{1:t}$. Then $\model(y_{t+1} \mid y_t) \geq e^{-2\lambda}/|\Sigma|$ for all $y_{1:t+1}$. Thus, by using $O(e^{2\lambda} |\Sigma| \log(|\Sigma|/\delta)/\ep^2)$ queries to a conditional sampling oracle $\Osamp$, we can output an estimate $\hat \logit(y_{1:t+1})$ for any sequence $y_{1:t+1}$ which satisfies $|\hat \logit(y_{1:t+1}) - \logit_\model(y_{1:t+1})| \leq \ep$ with probability $1-\delta$. By scaling $\delta$ down appropriately so that all of the algorithm's queries lead to $\ep$-approximate estimates of the corresponding logits with high probability, we can thus obtain the guarantee of \cref{thm:approx-main} using only conditional sampling queries, with an overhead of $\poly(e^{\lambda}, |\Sigma|, \log(1/\delta), \ep^{-1})$. 
\end{remark}

\begin{remark}[Sampling at high temperatures]
  \label{rmk:high-temp}
Some language model APIs allow one to vary the \emph{temperature} of conditional sampling: in particular, we are allowed to choose some 
\emph{temperature} $\tau > 0$, and can, for any given choice of $y_{1:t}$, receive conditional samples $y_{t+1} \sim \softmax(\tau^{-1} \cdot \logit_\model(y_{t+1} = \cdot \mid y_{1:t}))$. By choosing the temperature $\tau$ to be sufficiently large, we can guarantee that the distribution from which $y_{t+1}$ is sampled has logits bounded by $1$, and thus we can apply the observation in \cref{rmk:conditional-sampling} with $\lambda = 1$; in particular, this ensures only a polynomial blowup in all relevant parameters. 
\end{remark}

\subsection{Implication: learning boolean functions with queries}
\label{sec:km}
We briefly discuss one classical setting where our algorithm gives nontrivial guarantees (and which also helps to elucidate some of the barriers to further strengthening our algorithm's guarantees; see \cref{sec:conclusions}). We consider the setting of learning an unknown function $f : \{0,1\}^n \to [-1,1]$ with respect to the uniform distribution $\MU_n = \Unif(\{0,1\}^n)$. In particular, we consider the setting where we can query $x \in \{0,1\}^n$, and receive $f(x)$ as the response. Our goal is to output some function $g : \{0,1\}^n \to\BR$ satisfying $\E_{x \sim \MU_n}[(f(x) - g(x))^2] \leq \ep$. We use standard notation involving the Fourier analysis of boolean functions, i.e., for $S \subset [n]$ we define $\chi_S(x) := (-1)^{\sum_{i \in S} x_i}$, and we have the Fourier expansion $f(x) = \sum_{S \subset [n]} \hat f(S) \chi_S(x)$ where $\hat f(S) = \E_{x \sim \MU_n}[f(x) \chi_S(x)]$.

A celebrated result in this area is the \emph{Kushilevitz-Mansour} algorithm \cite{kushilevitz1991learning} which shows the following: consider any function $f : \{0,1\}^n \to [-1,1]$ for which there exists some $d$-Fourier sparse function $f'$ (i.e., at most $d$ of the Fourier coefficients of $f'$ are nonzero) satisfying $\E_{\MU_n}[(f(x) - f'(x))^2] \leq \ep$. Then there is a randomized algorithm which uses time and queries $\poly(n, 1/\ep, \log(1/\delta), d)$ and which outputs a function $g$ satisfying $\E_{\MU_n}[(f(x) - g(x))^2] \leq O(\ep)$ with probability at least $1-\delta$.

Our main result (\cref{thm:approx-main}) implies a weaker version of the Kushilevitz-Mansour theorem, in the following manner: 
\begin{corollary}
\label{cor:km-weak}
Suppose that $f : \{0,1\}^n \to [-1,1]$ is $\epavg$-approximated by a $d$-Fourier sparse function $f'$ in $L_2$, as described above. Then \cref{alg:learning-approx} can be used to produce a function $g$ satisfying $\E_{\MU_n}[(f(x) - g(x))^2] \leq O(\ep)$ using time and queries $\poly(n, 1/\ep, 1/\delta, d)$ with probability at least $1-\delta$, assuming that $\ep \geq \poly(\epavg, n, d, \log(1/\delta))$. 
\end{corollary}
\begin{proof}
Suppose we are  given $f,f'$ as above, we set $\Sigma = \{0,1\}$, $T = n+1$, and define a distribution $\model' \in \Delta(\Sigma^T)$ as follows: the distribution of $x_{1:n+1} \sim \model'$ is uniform on $x_{1:n}$, and given $x_{1:T-1} = x_{1:n}$, we define
\begin{align}
\model'(x_{n+1} = 1 \mid x_{1:n}) &= \sigma(f'(x_{1:n})) := \frac{\exp(f'(x_{1:n}))}{1 + \exp(f'(x_{1:n}))} \label{eq:mprime-last},
\end{align}
where $\sigma$ denotes the sigmoid function. 
We claim that $\model'$ is of logit rank $d$: indeed, we will in fact show it is an ISAN (\cref{def:isan}). Let us write $f'(x) = \sum_{i=1}^d \hat f'(S_i) \cdot \chi_{S_i}(x)$, for some subsets $S_1, \ldots, S_d \subset [n]$. Then we set:
\begin{align}
\mu = \mathbf{1} \in \BR^d, \qquad \BA_{t,0} = I_d, \qquad \BA_{t,1} = \mathrm{diag}(((-1)^{\mathbf{1}({t \in S_i})})_{i \in [d]}) \qquad \forall t \in [n],\nonumber\\
 \qquad \BB_{t,0} = \mathbf{0} \ \ \forall t \in [n], \quad \BB_{n+1,0} = \mathbf{0}, \quad \BB_{n+1,1} = (\hat f'(S_i))_{i \in [d]} \nonumber.
\end{align}
(Above $\BB_{t,0}, \BB_{t,1}$ denote the first and second rows of $\BB_t$.) Indeed, the above definition ensures that for any $x \in \{0,1\}^n$,
\begin{align}
\BA_{n,x_n} \cdots \BA_{1,x_1} \mu = (\chi_{S_i}(x))_{i \in [d]}, \nonumber
\end{align}
which implies that the last-token distribution of this ISAN coincides with that of $\model'$ in \cref{eq:mprime-last}. Moreover, the assumption that $\E[(f-f')^2] \leq \epavg$ implies that $\model$ has $O(\epavg)$-approximate $O(1)$-bounded logit rank $d$, per \cref{def:avg-closeness-true}. Thus, \cref{thm:approx-main} together with the condition $\ep \geq \poly(\epavg, n, d, \log(1/\delta))$ implies that, using polynomial time and queries, we can learn a model $\hat\model$ satisfying $\tvd{\hat\model}{\model'} \leq \ep$ with probability at least $1-\delta$.

Moreover, it is straightforward to check that $\tvd{\hat\model}{\model'} \leq \ep$ implies 
\begin{align}
    \E_{x_{1:n} \sim \MU_n}[\tvd{\hat \model(x_{n+1} = \cdot \mid x_{1:n})}{\model'(x_{n+1} = \cdot \mid x_{1:n})}] \leq O(n\ep),\nonumber
\end{align} 
which in turn implies that, letting $g(x_{1:n}) := \min \{1, \max\{-1,\sigma^{-1}(\hat \model(x_{n+1} = 1 \mid x_{1:n}))\}\}$, we have $\E_{\MU_n}[(f(x) - g(x))^2] \leq O(n\ep)$. Thus, applying \cref{thm:approx-main} to learn $\hat \model$ to accuracy $O(\ep/n)$ yields a function $g$ satisfying $\E_{\MU_n}[(f(x) - g(x))^2] \leq \ep$.
\end{proof}

\section{Conclusions and Future Directions} \label{sec:conclusions}

Modern language models are engineering marvels surpassing all previous expectations while also being largely a mystery from a theoretical perspective. 
Given their increasing prevalence in a variety of applications ranging from the mundane to the critical, it is important to develop systematic approaches for understanding them. 
We believe that perspectives from learning theory and theoretical computer science at large are instrumental in this endeavor, and we see our work as an initial step in this direction.
As such, we hope that our work will inspire further research exploring the theoretical foundations of large language models. 
Some promising directions for future work are listed below:

\paragraph{Improved polynomial dependence.} As discussed above, one gap between the assumptions needed for \cref{thm:approx-main} and the empirical behavior we observe (e.g., in \cref{fig:approx}) is that \cref{thm:approx-main} requires $\epavg < d^{-9.5}$ to obtain nonvacuous guarantees, while empirically we observe $\epavg \approx d^{-0.1}$. Can we improve our upper bound? We are not even sure if it is possible to avoid polynomial dependence on $d$ altogether:
\begin{problem}
    Is there an algorithm for learning low logit rank models in the setting of \cref{thm:approx-main} that only requires $\ep^\st \geq \Omega(\poly(\epavg, T, |\Sigma|, \alpha,1/\delta))$, i.e., without any polynomial dependence on the rank $d$?
\end{problem}

\paragraph{Learning from conditional samples.}
Some proprietary LLM APIs do not give full logit query access, but still allow conditional sampling access to $\model$ (see \cref{rmk:conditional-sampling,rmk:high-temp}). We ask if it is possible to obtain our results only under this weaker access (without suffering exponential dependence on the value of the logits, as in \cref{rmk:conditional-sampling}):
\begin{problem}
    \label{prob:conditional-sampling}
Can we learn (approximately) low-logit rank models to error $\ep$ in $\poly(T,d,|\Sigma|, \alpha, 1/\delta, 1/\ep)$ time using only a conditional sampling oracle?
\end{problem}
A positive answer to \cref{prob:conditional-sampling} would have a surprising implication, using the connection to learning boolean functions in \cref{sec:km}. In particular, consider any function $f : \{0,1\}^n \to \{-1,1\}$ of the form $f(x) = \mathrm{sign}(f'(x) + a)$, where $f'(x) = \sum_{i=1}^d a_i \cdot \chi_{S_i}(x)$, for some integers $a,a_i$ and sets $S_i \subset [n]$. Using the connection with low logit-rank models discussed in the proof of \cref{cor:km-weak}, we can simulate queries to $f(x)$ by using a conditional sampling oracle for a model $\model'$ defined in terms of $f'$ as in \cref{sec:km}, except with its last-layer logits scaled up by $\poly(n)$, so that with overwhelming probability $\model'(x_{n+1} = \cdot \mid x_{1:n})$ puts all its mass on $f(x_{1:n})$. Thus, a positive answer to \cref{prob:conditional-sampling} (with exponentially small approximation error) would allow us to efficiently learn \emph{threshold functions of sparse polynomials} using ony queries, which appears to be an open question in the literature. 

\paragraph{Latent variable models \& interpretability.} Latent variable models have long served as a powerful frameworks for interpretablity in data analysis both in the natural and social sciences. Starting  from the latent variable perspective on language models developed in this work, it is an interesting direction to explore interpretability of large language models through this lens. 
\paragraph{Perspectives from control theory.} Given the rich theory of control and linear dynamical systems, it is natural to ask whether the low logit rank (and the related ISAN model of \cref{def:isan}) perspective could be used to develop control-theoretic approaches to post-training and ensuring safety of large language models.

\paragraph{Empirical implications for sampling.} Another interesting direction is to explore the implications of low logit rank structure for sampling, for example to hide query access patterns or improving inference time computation.

\newpage

\section*{Acknoweldgements}
AL was supported by a Miller Research Fellowship.

\bibliographystyle{alpha}
\bibliography{bibliography.bib}

\end{document}